\documentclass{article}

  \PassOptionsToPackage{sort&compress}{natbib}

\usepackage[final]{neurips_2025}

\usepackage[utf8]{inputenc} 
\usepackage[T1]{fontenc}  
\usepackage{hyperref}    
\usepackage{url}      
\usepackage{booktabs}    
\usepackage{amsfonts}    
\usepackage{nicefrac}    
\usepackage{microtype}   
\usepackage{xcolor}     
\usepackage{adjustbox}
\usepackage{multirow,bigdelim}
\usepackage{algpseudocode}
\usepackage{enumitem}

\usepackage{header}
\usepackage{math_commands}

\usepackage{wrapfig}

\usepackage{geometry}
\usepackage{enumitem}
\usepackage{tabularx, caption}
\makeatletter
\newcommand*{\compress}{\@minipagetrue}
\makeatother

\usepackage{xparse}
\usepackage{tikz}
\usetikzlibrary{calc}
\usetikzlibrary{decorations.pathreplacing}

\newcommand*{\BraceAmplitude}{0.4em}%
\newcommand*{\VerticalOffset}{0.5ex}%
\newcommand*{\HorizontalOffset}{-0.4em}%
\NewDocumentCommand{\InsertLeftBrace}{%
  O{} 
  O{\HorizontalOffset,\VerticalOffset} 
  m  
  m  
  m  
}{%
  \begin{tikzpicture}[overlay,remember picture]
    \coordinate (Brace Top)  at ($(#3.north) + (#2)$);
    \coordinate (Brace Bottom) at ($(#4.south) + (#2)$);
  \draw [decoration={brace,mirror,amplitude=\BraceAmplitude}, decorate, thick, draw=blue, #1]
      (Brace Top) -- (Brace Bottom) 
      node [pos=0.5, anchor=east, align=left, text width=1.45cm, color=black, xshift=\BraceAmplitude] {#5};
  \end{tikzpicture}%
}%

\usepackage{makecell}
\usepackage{mwe}
\usepackage[most]{tcolorbox} 
\newcounter{myboxcounter}

\newtcolorbox{mybox}[1]{
 colback=bg,
 colframe=blue!75!black,
 fonttitle=\bfseries,
 title=#1,
 label=mybox:\themyboxcounter
}

\newtcolorbox{ttbox}[1]{
 colback=white,
 colframe=red!75!black,
 fonttitle=\bfseries,
 title=#1
}

\newtcolorbox{uutbox}[1]{
 colback=white,
 colframe=green!50!black,
 fonttitle=\bfseries,
 title=#1
}

\colorlet{bg}{white}

\usepackage[textwidth=18mm]{todonotes}

\newif\ifappendix
\appendixfalse

\title{ Descriptive History Representations: \\ Learning Representations by Answering Questions }
\author{Guy Tennenholtz$^\dagger$\thanks{Correspondence to: \texttt{guytenn@gmail.com}},~~Jihwan Jeong$^\dagger$,~~Chih-Wei Hsu$^\dagger$,~~Yinlam Chow$^\ddagger$,~~Craig Boutilier$^\dagger$
\\~\\
\large $\dagger$ Google Research, $\ddagger$ Google Deepmind}

\begin{document}

\doparttoc
\faketableofcontents

\maketitle

\begin{abstract}
  Effective decision making in partially observable environments requires compressing long interaction histories into informative representations. We introduce \emph{descriptive history representations (DHRs)}: sufficient statistics characterized by their capacity to answer \emph{relevant questions} about past interactions and potential future outcomes. DHRs focus on capturing the information necessary to address task-relevant queries, providing a structured way to summarize a history for optimal control. We propose a multi-agent learning framework, involving representation, decision, and question-asking components, optimized using a joint objective that balances reward maximization with the representation's ability to answer informative questions. This yields representations that capture the salient historical details and predictive structures needed for effective decision making. We validate our approach on user modeling tasks with public movie and shopping datasets, generating interpretable textual \emph{user profiles} which serve as sufficient statistics for predicting preference-driven behavior of users.
\end{abstract}

\section{Introduction}
\label{section: introduction}

Reinforcement learning (RL) in partially observable environments is challenging: agents must make effective decisions with incomplete information, making optimal decisions dependent on the complete observation-action history \citep{aastrom1965optimal,sondik1971optimal,kaelbling1998planning}.
As such, representing the history in a compact, informative way for decision making is crucial. Traditional approaches in partially observable Markov decision processes (POMDPs) use \emph{belief states}, i.e., distributions over an underlying state space \citep{aastrom1965optimal,sondik1971optimal,kaelbling1998planning}, but require a pre-specified Markovian state space.
Conversely, directly using the entire history becomes computationally intractable as it grows. 
\emph{Predictive state representations (PSRs)} \citep{littman2001predictive,boots_PSRs:IJRR2011} address this by representing the state as a set of predictions about future observations, conditioned on specific action sequences.

In this work, we introduce \emph{descriptive history representations (DHRs)}, a framework that learns representations that focus on \emph{answering questions}. Instead of predicting low-level observations, we construct history representations that answer broad classes of predictive questions. These questions, ideally formulated flexibly (e.g., in natural language), aim to capture essential information for effective decision-making. DHRs shift the representation burden away from predicting specific, low-level observations to a higher level of abstraction, ensuring they answer \emph{relevant} queries. 

As an example, consider a recommender agent that interacts with users. A \emph{user profile}, representing a user's past interactions with the agent, should suffice to answer queries such as: ``Does the user prefer Item 1 to Item 2?''; ``How likely is the user to abandon the session?''; or open queries like  ``Write a review this user might provide for the following product: <description>.'' Such queries can be far more natural, task-relevant, easier to specify, and reason about, than the precise sequence of low-level future observations (e.g., specific interactions, sensor readings, pixel values) that simply correlate with user satisfaction. Queries of this form are the foundation of DHRs.

To \emph{learn} DHRs, we employ a multi-agent cooperative training paradigm. A \emph{representation (DHR) encoder} learns to construct a summary of the history, while an \emph{answer agent} learns to generate answers to queries about the history / future using \emph{only} the generated summary. Finally, a \emph{decision agent} determines the next action to take given the current DHR summary. The agents' objectives are aligned with expected reward maximization in the underlying environment. 

Our key contributions are as follows. (1) We introduce DHRs, defined by their ability to answer relevant questions about past interactions and future outcomes. We formally establish them as \emph{sufficient statistics} (formally defined in Section \ref{section: problem formulation}) for effective decision-making. (2) We propose a multi-agent framework for learning DHRs, using representation, decision, and question-asking components, jointly optimized to balance reward maximization with the representation’s ability to answer informative questions. (3) We demonstrate the efficacy of the learned DHRs on recommendation domains, using public movie and shopping datasets, showcasing their ability to generate predictive textual user profiles and improve recommendation quality.

\section{Problem Setting} 
\label{section: problem formulation}

We consider a \emph{partially observable environment (POE)} defined by the tuple $(\gO, \gA, T)$, where $\gO$ is the observation space, $\gA$ is the action space, and $T: (\gO \times \gA)^* \mapsto \Delta_{\gO}$ is the transition function, mapping (action-observation) histories to a distribution over observations. We assume a \emph{reward function} over histories $R: (\gO \times \gA)^* \mapsto \sR$. At each time $t$, the environment is in a history state $h_t = (o_1, a_1, \hdots, o_t)$. An agent takes an action $a_t \in \gA$, receives reward $r_t = R(h_t, a_t)$, and the environment transitions to a new state via $o_{t+1} \sim T(\cdot |h_t, a_t)$.

A \emph{policy} $\pi: (\gO \times \gA)^* \times \gO \mapsto \Delta_{\gA}$ maps histories to distributions over actions. Let $\Pi$ denote the set of policies. The \emph{value} of policy $\pi \in \Pi$ is its expected cumulative return:
$
    V(\pi) = \mathbb E\big[\sum_{t=1}^H r(h_t, a_t) | a_t \sim \pi(h_t)\big],
$
where $H$ is the horizon. 
Our objective is to find a policy $\pi^*$ with maximum value, i.e.,
$
    \pi^* \in \arg\max_{\pi} V(\pi).
$

We use the notion of an \emph{$f$-sufficient statistic} below. Let $f: \gH \times \Pi \mapsto \sR$ be a specific function of interest. We say a mapping $E: \gH \mapsto \gZ$ is an $f$-sufficient statistic of $\gH$ if, for any $h, \pi$, there is some $\tilde{f}: \gZ \times \Pi \mapsto \sR$ such that $f(h; \pi) = \tilde{f}(E(h); \pi)$ \citep{remon1984concept}. Examples include value-sufficient statistics for POMDPs such as belief state representations \citep{aastrom1965optimal}, and PSRs \citep{littman2001predictive}.
The main task of our work is to learn such a history-to-state mapping $E$ that both effectively summarizes historical information and is amenable to sequential decision making.

Given a probability function $y \sim q(y|x)$, we
often write $y \sim q_x$.
Let $\gH_t$ be the set of histories of length $t$, and $\gH = \cup_{t=1}^H \gH_t$. Similarly, $\Omega_t$ is the set of future realizations
beginning at time $t$, with $\Omega = \cup_{t=1}^H \Omega_t$. Thus
$(h_t, \omega_t) \in \gH \times \Omega$ such that $h_t = \brk*{o_1, a_1, \hdots, o_t}$ and $\omega_t = \brk*{a_t, o_{t+1}, \hdots, o_H}$.
For any $\pi\in\Pi$, $d^\pi(\omega | h)$ is the probability of a future $\omega$ given $\pi$ starting at $h$.

\section{Descriptive History Representations (DHRs)}

This section formally defines DHRs. We first introduce the \emph{Question-Answer-space (QA-space)}, which structures historical information through questions and answers, and then detail how DHRs serve as compact, actionable summaries derived from these spaces.

\subsection{QA-Spaces}

Informally, a \emph{QA-space} comprises questions and answers over histories. In a conversational recommender, given a user's interaction history, useful questions might query latent preferences (e.g., ``which brand does the user prefer?''), with answers being (distributions over) preference attributes (e.g., ``the user prefers brand A over B with probability $0.8$"). Alternatively, a question could probe user behavior (e.g., ``Write a review this user might provide for item A''), answered by a distribution over possible reviews. Formally, a QA-space defines questions and answers which adhere to a particular functional form, requiring answers to be outcomes of questions.
\begin{definition}[QA-space]
    A QA-space is a tuple $(\gQ, \gY, \gX, \nu)$, where $\gQ$ is a question space, $\gY$ is an answer space, $\gX$ is a context space, and $\nu: \gX \times \gQ \mapsto \Delta_{\gY}$ is an answer function.\footnote{When $\gY$ is continuous, $\nu$ maps questions and contexts to 
    probability measures on the Borel sets of $\gY$.}
\end{definition}
While QA-spaces are general, in this work we focus on QA-spaces that ask \emph{semantically meaningful} and \emph{interpretable} questions about histories. Thus, the context space $\gX$ is the space of histories~$\gH$, and the answer function $\nu$ poses informative questions about those histories (e.g., summarization of past interactions, or predictions of future events). Notice that the answer function $\nu$ must generally return a \emph{distribution} over the answer space given the partially observable nature of the environment.

We are primarily interested in QA-spaces where the collective answers to questions provide all necessary information for a given purpose (like decision-making or prediction), thereby acting as sufficient statistics. In the recommender setting, a QA-space with user preference questions can serve as a sufficient statistic for future behaviors (e.g., item acceptance, session abandonment) by encoding inferred latent preferences. This motivates the definition of a \emph{sufficient QA-space}, which asks questions that induce a sufficient statistic.
\begin{definition}[$f$-Sufficient QA-space]
    A QA-space $(\gQ, \gY, \gX, \nu)$ is $f$-sufficient, if for any $x \in \gX$, there is a subset $Q^*_x \subseteq \gQ$ such that $\brk[c]*{\brk*{q, \nu(x,q)}}_{q \in Q^*_x}$ is an $f$-sufficient statistic. We call $Q^*_x$ the set of ($f$-) sufficient questions for context $x$, and $Q^* = \cup_{x \in \gX} Q^*_x$ the set of sufficient questions.
\end{definition}
One example is the \emph{value-sufficient} QA-space, which enables prediction of the value ${V^\pi(h) = \expect*{}{\sum_{t=k}^H r(h_t, a_t) | h_k=h}}$ of any policy $\pi$ and history $h \in \gH$. 

\begin{figure}[t!]
    \centering
    \includegraphics[width=.9\textwidth]{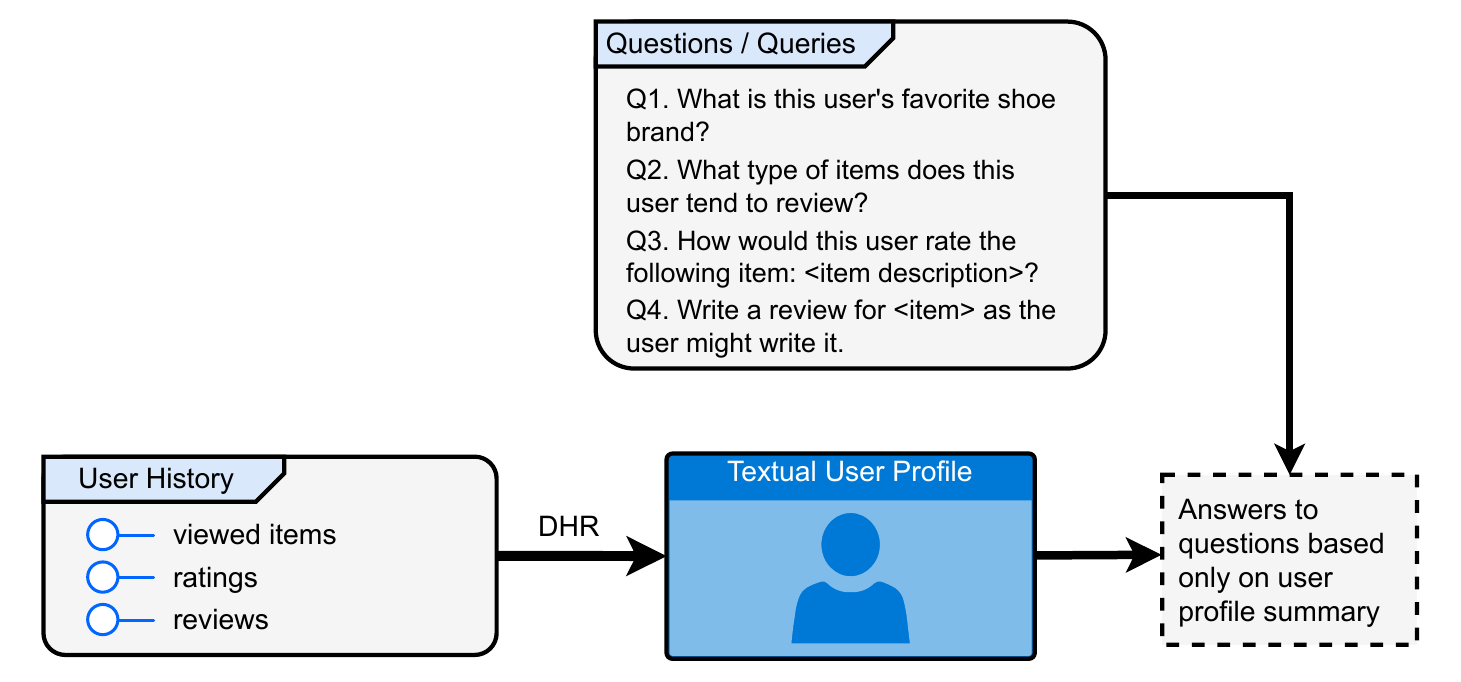}
    \caption{\footnotesize An illustrative example of a descriptive history representation (DHR) which maps a user's history to a compact user profile, which is sufficient for answering questions about the user's preferences.}
\label{fig: user profile}
\end{figure}

\subsection{Examples of QA-Spaces}

QA-spaces encompass a variety of representation techniques, including classical POMDP methods.

\textbf{Belief States (\citet{kaelbling1998planning}).} 
For POMDPs with state space $\gS$, belief states $b(s|h)$ are distributions over states which constitute a sufficient statistic for value prediction. We define a QA-space for belief states as follows. Let $\gQ = \gS$, $\gY = [0,1]$, and answer function $\nu(h,s) = b(s|h)$. For any history $h \in \gH$, the set of question-answer pairs $\brk[c]*{(s, b(s|h))}_{s \in \gS}$ constitutes a sufficient QA-space, since $b$ is a sufficient statistic of the history.

\textbf{Predictive State Representations (PSRs, \citet{littman2001predictive}).} 
A PSR is defined by a vector of probabilities over a set of core tests $\brk*{P(\omega^1 | h), \hdots, P(\omega^k | h)}$, for any history $h\in\gH$. Like belief states, PSRs are sufficient statistics. To construct a QA-space for PSRs, let $\gQ$ be the set of tests, $\gY = [0,1]$ as before, and define an answer function as $\nu(h,\omega) = P(\omega|h)$.

\textbf{Item Recommendation (informal).} Consider a user interacting with a recommender, where actions are recommended item slates, and observations are user interactions (e.g., choices, ratings, etc.). The recommender exploits a \emph{user profile} representing user preferences learned from interactions (see \Cref{fig: user profile}). Given a user's history and candidate items, useful questions assess preferences over pairs of items, or the likelihood of accepting recommendations. A sufficient statistic includes answers which are predictive of preferences, behavior, and suffice for optimal recommendations.

\subsection{Descriptive History Representations of QA-Spaces} 
\label{section: DHR}

While QA-spaces offer a powerful mechanism for identifying vital information within a history~$h$ by finding answers $\nu(h, q)$ to a set of sufficient questions, they do not inherently prescribe how this distilled information should be structured into an effective representation. 
We thus seek to leverage the insights obtained from QA-spaces to learn a condensed representation, which retains the information required to answer \emph{sufficient} questions. 

\begin{definition}[Descriptive History Representation]
\label{def: DHR}
    Let $(\gQ, \gY, \gH, \nu)$ be an $f$-sufficient QA-space. An embedding $E: \gH \mapsto \gZ$ is called a Descriptive History Representation (DHR) if there exists $\nu_A: \gZ \times \gQ \mapsto \Delta_{\gY}$ such that $\nu_A(z, q) = \nu(h, q)$, for any $h \in \gH$, and any sufficient question $q \in Q^*_h$.
\end{definition}

To this end, a DHR encapsulates the essence of the history for question answering by learning a representation $E : \mathcal{H} \mapsto \gZ$ that maps a history $h \in \mathcal{H}$ to a compact representation $z \in \gZ$. In general, $\gZ$ can be arbitrary; e.g., if $\gZ \subseteq \sR^d$, then $E$ is a classic embedding. In our recommendation example, $z \in \gZ$ could be a \emph{textual user profile}, describing a user's preferences in natural language (see \Cref{fig: user profile} for an illustration). The central idea of a DHR  is that the representation acts as a proxy for the history, preserving precisely the information required to answer the relevant questions defined by the QA-space. Thus, instead of computing $\nu(h, q)$ from the history $h$, a \emph{compressed} answer function $\nu_A : \gZ \times Q \to \Delta_Y$ produces the same answer distribution using only $z = E(h)$. 
For example, if the textual user profile 
mentioned above is a DHR for recommendation, it should allow one to predict a user's ranking over new items, and other relevant behaviors, without access to the raw history. 


A DHR, as denoted by $z = E(h)$, is a powerful concept because it effectively compresses a potentially complex history $h$ into a more compact form $z$. By preserving the means to answer sufficient questions, the DHR encapsulates all task-relevant information defined by the QA-space, suggesting that it is a sufficient statistics of a given task. Indeed, in the following proposition, we establish this connection (see \Cref{appendix: proof of dhr proposition} for proof).
\begin{restatable}{proposition}{dhrsufficientproposition}
\label{proposition: DHR is sufficient}
    Let $E: \gH \mapsto \gZ$ be a DHR (\Cref{def: DHR}). Then it is also an $f$-sufficient statistic.
\end{restatable}
Having established the DHR and its connection to sufficient statistics, in the following section we leverage this framework to address partially observable RL problems. 

%

\begin{figure}[t!]
    \centering
    \includegraphics[width=0.95\textwidth]{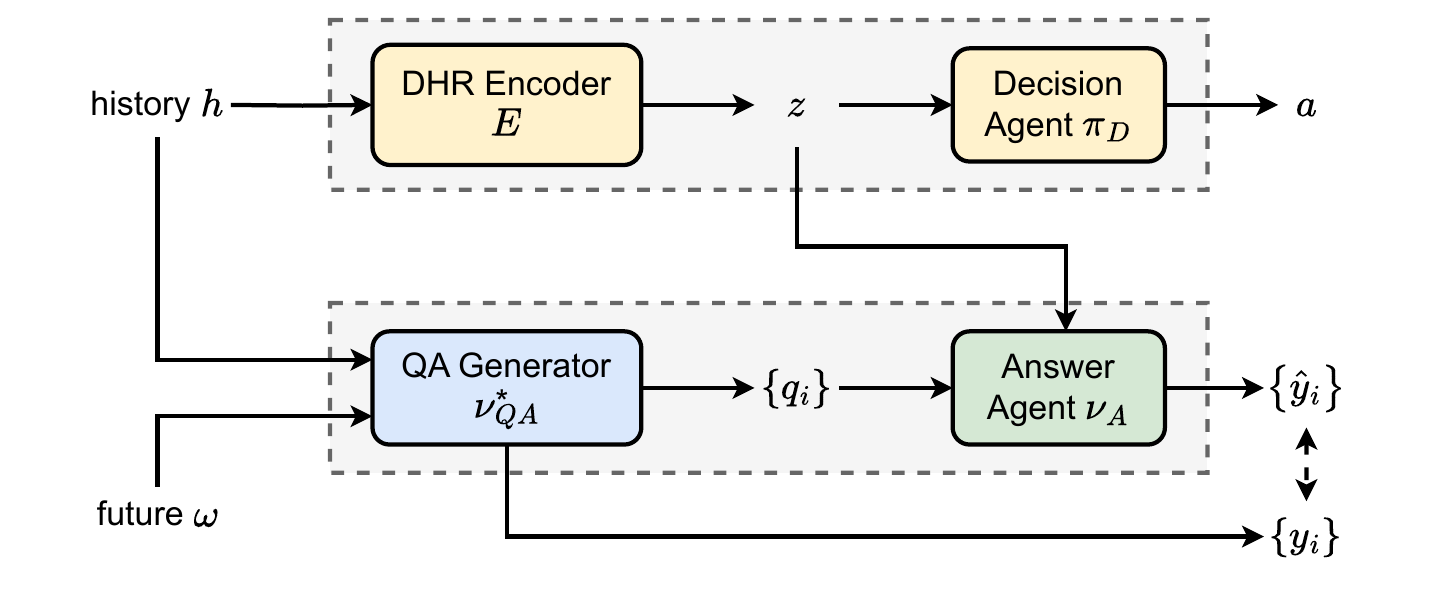}
    \caption{\footnotesize An illustration of our learning framework. A policy $\pi$ is composed of a DHR encoder $E$ and a decision policy $\pi_D$. The DHR embedding takes a history $h$ and creates a representation $z$. The decision policy takes actions in the environment given the representation $z$. In order to learn a DHR ($E$), we use a QA generator, responsible for generating history ($h$) and future ($\omega$)-dependent question-answer pairs. Finally, an answer agent is responsible for answering said questions with access to only the representation $z$. The training procedure is described in \Cref{section: DHR} and \Cref{alg: DHRL}.}
\label{fig: dhr framework}
\end{figure}

\begin{algorithm}[t!]
\caption{Descriptive History Representation Learning (DHRL)}
\label{alg: DHRL}
\begin{algorithmic}[1] 
    \State Initialize $\theta_E, \theta_A, \theta_D, \theta_g$.
    \For{each training iteration}
        \State Sample trajectory $\tau = \brk*{o_1, a_1, \hdots, o_H}$ and rewards $r_1, \hdots, r_H$. \Comment{Using $E, \pi_D$}
        \For{$t=1, \hdots, H$}
            \State Split trajectory to history $h_t$ and future $\omega_t$.
            \State Sample QA pairs $\{(q_{kt}, y_{kt})\}_{k=1}^K \sim \nu^*_{QA}(h_t, \omega_t)$. 
            \State Sample representation $z_t \sim E(h_t)$.
            \State Predict answers $\hat{y}_{kt} = \nu_A(z_t, q_{kt})$ for $k=1 \hdots K$.
        \EndFor
        \State Update $\theta_g$ according to $\frac{1}{HK} \sum_{t=1}^H \sum_{k=1}^K \nabla_{\theta_g} \brk[s]*{  f^*(g_{\theta_g}(h_t, q_{kt}, \hat{y}_{kt})) - g_{\theta_g}(h_t, q_{kt}, y_{kt})}$. \label{alg:update_g}
        \State Update $\theta_D, \theta_A, \theta_E$ via RL using rewards $r^D, r^A, r^E$  (See \Cref{section: dhrl}). \label{alg:update_policy}
    \EndFor
\end{algorithmic}
\end{algorithm}

\section{Learning Descriptive History Representations}
\label{section: dhrl}

Our goal is to learn a DHR that is effective for downstream decision making. Specifically, we learn a \emph{DHR policy} $\pi = (E, \pi_D)$ consisting of two components: a \emph{DHR embedding} $E: \gH \mapsto \Delta_{\gZ}$ (i.e., a history \emph{encoder}); and 
 a \emph{decision policy} $\pi_D: \gZ \mapsto \Delta_{\gA}$ over summarized histories. The DHR policy induces a policy over histories in the obvious way: $\pi(h) = \expect*{z \sim E(h)}{\pi_D(z)}$. To ensure $E$ is indeed a DHR (i.e., a sufficient statistic), we train an \emph{answer function} $\nu_A: \gZ \times \gQ \mapsto \gY$ for any sufficient question $q \in Q^*_h$ (see \Cref{def: DHR}). We detail the training objective and methods in this section.
 
\paragraph{Generating Sufficient Questions and Answers.}
 To train $\nu_A$, we need to define the set of sufficient questions $Q^*$. To do this, we introduce an oracle \emph{QA-generator}---only required during training---which provides ground-truth question-answer pairs $(q,y)$ that define the content used to train the DHR. To obtain a QA-generator, we use full trajectory realizations; i.e., a trajectory $\tau:=(h, \omega)$, where $h$ is the history, and $\omega$ is its future realization. Here, the QA generator leverages both $h$ and $\omega$ to construct questions and their corresponding ground-truth answers.
 
 Designing an oracle QA-generator to produce sufficient questions can be challenging. In  \Cref{sec:experiments}, we detail how we design a QA-generator using large language models (LLMs) and future user interactions (i.e., $\omega$) for recommendation tasks. More generally, an LLM-based QA-generator can analyze $h$ and a specific future realization $\omega$ to generate a template question about a specific class of events in $\omega$, with the answer directly derived from $\omega$. We refer to \Cref{appendix: qa generator} for a detailed discussion on the construction and training of a QA-generator (e.g., through adversarial training). 
 
 Formally, a QA generator is a mapping $\nu^*_{QA}: \gH \times \Omega \mapsto \Delta_{\gQ \times \gY}$. Its purpose is to create sufficient questions and their answers using a history and its corresponding future. For clarity (with a slight abuse of notation), we decompose a QA generator's output distribution as ${\nu^*_{QA}(h, \omega) = \nu^*_A(y | q, h, \omega)\nu^*_Q(q | h, \omega)}$. Importantly, when training our answer agent $\nu_A$, we assume access to samples from $\nu^*_{QA}$.

\paragraph{Training Objective.}
We define the joint objective of the DHR encoder, answer agent, and decision agent to maximize value and match the distribution of ground-truth answers (i.e., match $\nu^*_A$):
\begin{align}
\label{eq: DHR primal objective}
    \max_{E, \nu_A, \pi_D} 
        (1-\lambda) V(\pi) 
        - 
        \lambda 
        \fdiv{d^{\nu^*_A}}{d^{\nu_A}},
\tag{OPT 1}
\end{align}
where the hyper-parameter $\lambda \in [0,1]$ balances between RL and DHR learning, and
\begin{align*}
    d^{\nu_A}(y,h,q) &= P\brk*{y,h,q \mid q \sim \nu^*_Q(h, \omega), y \sim \nu_A(z, q), z \sim E(h), \omega \sim d^\pi(\omega \mid h)}, \\
    d^{\nu^*_A}(y,h,q) &= P\brk*{y,h,q \mid q,y \sim \nu^*_{QA}(h, \omega), \omega \sim d^\pi(\omega \mid h)}.
\end{align*}
Here, $D_f$ denotes a specific $f$-divergence (e.g., TV-distance, $\chi^2$-divergence, KL-divergence). 

Directly optimizing \Cref{eq: DHR primal objective} is challenging, primarily due to the difficulty in accurately modeling the distributions $d^{\nu_A}$ and $d^{\nu^*_A}$. Instead, we solve its variational form:
\begin{align}
    \label{eq: dual improved}
    \max_{\pi_D, E, \nu_A}
    \min_{g: \gH \times \gQ \times \gY \mapsto \sR} 
    \expect*{}{
        (1 - \lambda) r(h, a) 
        +
        \lambda
        \expect*{y \sim \nu_A(q,h,\omega)}{f^*(g(h,q,y))}
        - 
        \lambda 
        \expect*{y \sim \nu^*_A(z, q)}{g(h,q,y)}
    },
    \tag{OPT 2}
\end{align}
where $f^*$ is the convex conjugate of the convex function $f$ defining the divergence, and $\mathbb{E}$ is taken w.r.t. $h, \omega \sim d^\pi, q \sim \nu^*_Q(h,\omega), z \sim E(h)$. Instead of directly estimating and matching question-answer-history distributions we use samples from $\nu_A$, the DHR's (current) answers, and $\nu^*_{A}$, the QA-generator's ground-truth answers, to train a discriminator $g$. The max-min objective in \Cref{eq: dual improved} can be solved efficiently by iteratively training the discriminator $g$ and the agents ($\pi_D, E, \nu_A$), as we discuss next.

\paragraph{DHR Learning (DHRL).}

We now detail the \emph{DHR learning (DHRL)} algorithm (see \Cref{alg: DHRL}). The complete learning framework is depicted in \Cref{fig: dhr framework}. DHRL trains the DHR encoder, answer agent, decision agent, and discriminator, with parameters $\theta_E, \theta_A, \theta_D, \theta_g$, respectively. It uses discriminator-based learning per \Cref{eq: dual improved}. DHRL first samples trajectories using $E$ and $\pi_{D}$ (line 3). These trajectories are used in hindsight: each trajectory is split into histories $h_t$ and their corresponding future realizations $\omega_t$ (line 5). The QA-generator $\nu^*_{QA}$ uses both $h_t$ and $\omega_t$ to generate ground-truth question-answer pairs $(q_k, y_k)$ (line 6). A representation $z_t \sim E(h_t)$ is sampled (line 7), and the answer agent predicts answers $\hat{y}_k = \nu_A(z_t, q_k)$ according to that representation (line 8). 

The collected batch data is used to update the discriminator $g$ (line 10) via gradient descent w.r.t. the objective $\expect*{}{  
\expect*{y \sim \nu^*_A(q,h,\omega)}{f^*(g(h,q,y))}
-
\expect*{y \sim \nu_A(z, q)}{g(h,q,y)}}$. The DHR encoder $E$, answer agent $\nu_A$ and decision agent $\pi_D$ are updated using RL (line 11), based on reward signals derived from objective \Cref{eq: dual improved}. Specifically, we update the decision agent $\pi_D$ using reward $r^D = r(h,a)$, the answer agent with reward $r^A = f^*(g(h, q, \hat{y}))$, and the DHR embedding with reward $r^E = (1-\lambda)R^D + \lambda R^A$. These updates drive the system toward a DHR that both maximizes rewards and provides accurate answers according to the QA distribution and the learned discriminator.

\section{Experiments}
\label{sec:experiments}

Our experiments are designed to validate DHRL, focusing on its ability to generate predictive representations that answer relevant questions. We demonstrate that DHRL can produce high-quality DHRs in recommendation domains, where we use DHRs to generate textual user profiles which (1) accurately answer questions about user preferences and future behavior; (2) induce informative, coherent summaries of user history; and (3) effectively support a downstream recommendation task.

\subsection{Experimental Setup.}

We use two datasets in our experiments: MovieLens 25M \citep{harper2015movielens}, and Amazon Reviews \citep{amazon_reviews} (specifically, the \emph{Clothing, Shoes and Jewelry} category),
consistent with prior work \citep{tennenholtz2024demystifying,tennenholtz2024embedding}.
User histories $h_t = (o_1, a_1, ..., o_t)$ consist of sequences of observations (item titles, ratings, and, for Amazon, descriptions, prices and reviews) and actions (recommended items). This history $h_t$ serves as input to the representation encoder $E$, which creates a textual user profile (see examples of user profiles in \Cref{appendix: qualitative results}). Future user interactions $\omega_t = (a_t, o_{t+1}, ..., o_H)$ are only used in hindsight to construct the question-answer pairs for training. 

\textbf{Question Generation.}\hspace*{1mm}
For both MovieLens and Amazon, questions $q_k$ compare pairs of held-out movies (e.g., ``Rank the following movies based on the user's preferences: [<movie\_1>, <movie\_2>]''), with the ground-truth answer $y_k$ taken from future user ratings in $\omega_t$. Multiple ranking questions are generated per history, each comparing a distinct movie pair (see \Cref{appendix: agent prompts,appendix: qualitative results} for examples of prompts, questions, and answer formats).
For ranking questions, models must adhere to a specific format (see \Cref{appendix: agent prompts}), and a negative reward is given for parsing failures. For Amazon Reviews, we augment the question set with review generation queries $q_{\text{review}}$ which ask ``Write a review for item <item description> as <user> might write it.'' Here, the ground-truth answer $y_{\text{review}}$ is the user's textual review of that item from
$\omega_t$. For exhaustive details see \Cref{appendix: QA generator experiment design}.

\textbf{Training Methodology.}\hspace*{1mm}
We train $E$ and $\nu_A$ via standard policy gradient (PG). We use Gemini 1.5 Flash \citep{team2024gemini} and Gemma V3 (4B, 12B) \citep{team2025gemma} as our language models. All agents use the same base model and share parameters, differentiated only by their  prompts. Training details and prompts used for training our agents are provided in \Cref{appendix: implementation details,appendix: agent prompts}. 

Our decision agent $\pi_D$ is fine-tuned to make item recommendations to the user given their generated profile (DHR). Specifically, the agent is given the profile (critically, it does not have access to the full user history), and a set of held-out items and is asked to recommend the best item in the set, receiving a reward (between 0 and 1) based on the user's rating for the item (higher ratings induce higher rewards).
Unless otherwise stated, we use a Gemma V3 4B model, a 256-token profile limit, $K=5$ item comparison questions, one review question (Amazon only), a history length of $H=10$ interactions, and TV-distance as our divergence. Models generally converge in under 1000 iterations, hence we report results based on 1000 training iterations.

\begin{table}[t!]
\definecolor{lightgray}{gray}{0.9} 
\centering
\caption{\footnotesize Performance of DHRL with default settings (profile with max. 256 tokens, 5 questions, 10 history interactions, TV-divergence) after 1,000 iterations for Gemma V3 4B, 12B, and Gemini 1.5 Flash models.}
\label{tab:main_results}
\begin{tabular}{lccccc}
\toprule
\bf Model \& Method & \bf \makecell{Prediction \\ Accuracy \\ (w.r.t. GT)} & \bf \makecell{Rec. \\ Reward \\ ($\boldsymbol{r^D}$)} & \bf \makecell{Profile-History \\ Consistency \\ (AI/Human)} & \bf \makecell{~ Prediction \\ Fidelity \\ (AI/Human)} & \bf \makecell{Review \\ Quality \\ (AI/Human)} \\
\midrule\midrule
\multicolumn{6}{c}{Task: recommendation using Amazon products user profiles} \\
\midrule
Gemma V3 4B & 0.34  & 0.54 & 3.46~/~3.28 & 0.44~/~0.47 & 0.15~/~0.27\\
\rowcolor{lightgray} Gemma V3 4B+DHRL  & 0.71 & 0.83 & 3.41~/~3.85 & 0.56~/~0.53 & 0.85~/~0.73\\
Gemma V3 12B & 0.67  & 0.78 & 4.32~/~4.18 & 0.34~/~0.39 & 0.17~/~0.35\\
\rowcolor{lightgray} Gemma V3 12B+DHRL  & 0.75 & 0.84 & 4.39~/~4.37 & 0.66~/~0.61 & 0.83~/~0.65\\
Gemini 1.5 Flash & 0.69  & 0.8 & 4.41~/~4.3 & 0.38~/~0.41 & 0.17~/~0.38\\
\rowcolor{lightgray} Gemini 1.5\ Flash+DHRL & 0.74  & 0.86 & 4.46~/~4.34 & 0.62~/~0.59 & 0.83~/~0.62\\
\midrule\midrule
\multicolumn{6}{c}{Task: recommendation using MovieLens user profiles} \\
\midrule
Gemma V3 4B & 0.37  & 0.61 & 3.72~/~3.96 & 0.32~/~0.45 & -\\
\rowcolor{lightgray} Gemma V3 4B+DHRL & 0.74 & 0.79 & 3.84~/~4.24 & 0.68~/~0.55 & -\\
Gemma V3 12B & 0.58  & 0.64 & 4.66~/~4.58 & 0.35~/~0.42 & -\\
\rowcolor{lightgray} Gemma V3 12B+DHRL  & 0.84 & 0.93 & 4.62~/~4.71 & 0.76~/~0.66 & -\\
Gemini 1.5 Flash & 0.62 & 0.68 & 4.69~/~4.53 & 0.38~/~0.39 & -\\
\rowcolor{lightgray} Gemini 1.5\ Flash+DHRL & 0.82 & 0.92 & 4.74~/~4.7 & 0.77~/~0.62 & -\\
\bottomrule
\end{tabular}
\end{table}

\textbf{Evaluation Metrics.}\hspace*{1mm}
The generated DHRs (user profiles $z$) and the accuracy of the answer agent $\nu_A$ are evaluated using three main criteria: prediction scores, recommendation quality, and qualitative AI and human evaluation. %
First, we measure \emph{predictive accuracy}. For ranking questions, which compare pairs of items, the answer agent $\nu_A$ predicts the user's preferred item given profile $z$. Predictive accuracy is the fraction of pairwise comparisons the agent predicts correctly (w.r.t.\ ground-truth future user ratings).
Second, we evaluate \emph{recommendation reward}, which assesses the utility of generated DHRs for our downstream task, i.e., a decision policy $\pi_D$ trained to use the user profile to recommend items. Given a DHR and a held-out set of unseen items, $\pi_D$ recommends an item from the set it predicts the user most prefers. Recommendation reward reflects the user's true rating for this item (normalized to $[0,1]$). 
Finally, we qualitatively evaluate our results using both AI and human feedback. For AI feedback, we use Gemini 2.5 Pro, and for human evaluation we employ 24 raters\footnote{Raters were paid contractors. They received their standard contracted wage, which is above the living wage in their country of employment.}. 

Our AI / human evaluation consists of three criteria:
\begin{enumerate}[wide, labelwidth=!, labelindent=0pt]
\item \emph{Profile-History Consistency}: Raters are given a user interaction history $h_t$ and its generated textual profile $z_t$. They rate (1-5 scale) how accurately and coherently $z_t$ summarizes the information in $h_t$.
\item \emph{Prediction Fidelity}: Raters are given a generated user profile $z_t$ and are asked to predict the user's preferences over a set of held-out item pairs
(mimicking the task of the answer agent). We compute overall rater prediction accuracy w.r.t.\ ground truth ratings. We report the win-rate between the rater prediction accuracy of the baseline method against its DHRL counterpart. Note that \emph{human performance on this task may be inherently limited}: predicting preferences from a textual summary with no prior training is challenging (potentially biasing our  results).
\item \emph{Review Quality} (Amazon only): Raters are shown the user profile $z_t$ and two reviews for a held-out user item, one generated by the baseline model (i.e., prior to DHRL training) and another by the DHRL model. Raters are asked which of the two reviews is more likely to have been written by the user (given the profile $z_t$). We report the win-rate between the DHRL-generated review and the baseline-generated review, indicating its perceived authenticity and relevance. Examples of generated user profiles and reviews are included in \Cref{appendix: qualitative results}.
\end{enumerate}

\begin{figure*}[t!]
    \centering
    \includegraphics[width=0.32\linewidth]{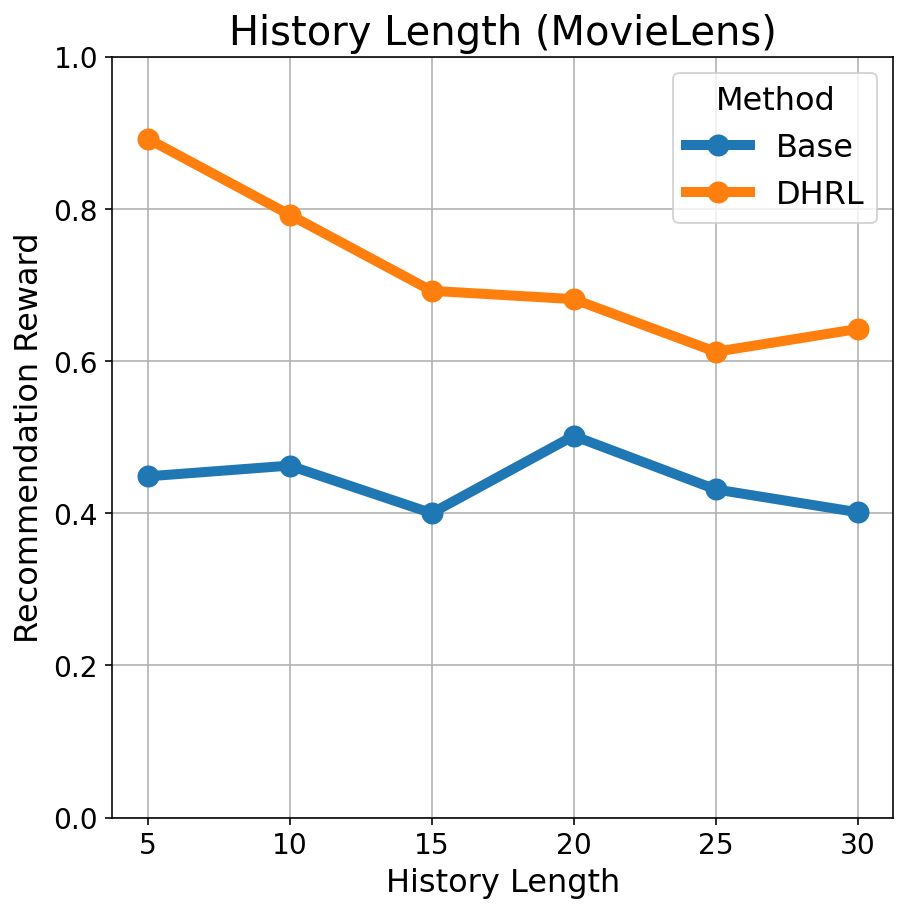}
    \includegraphics[width=0.32\linewidth]{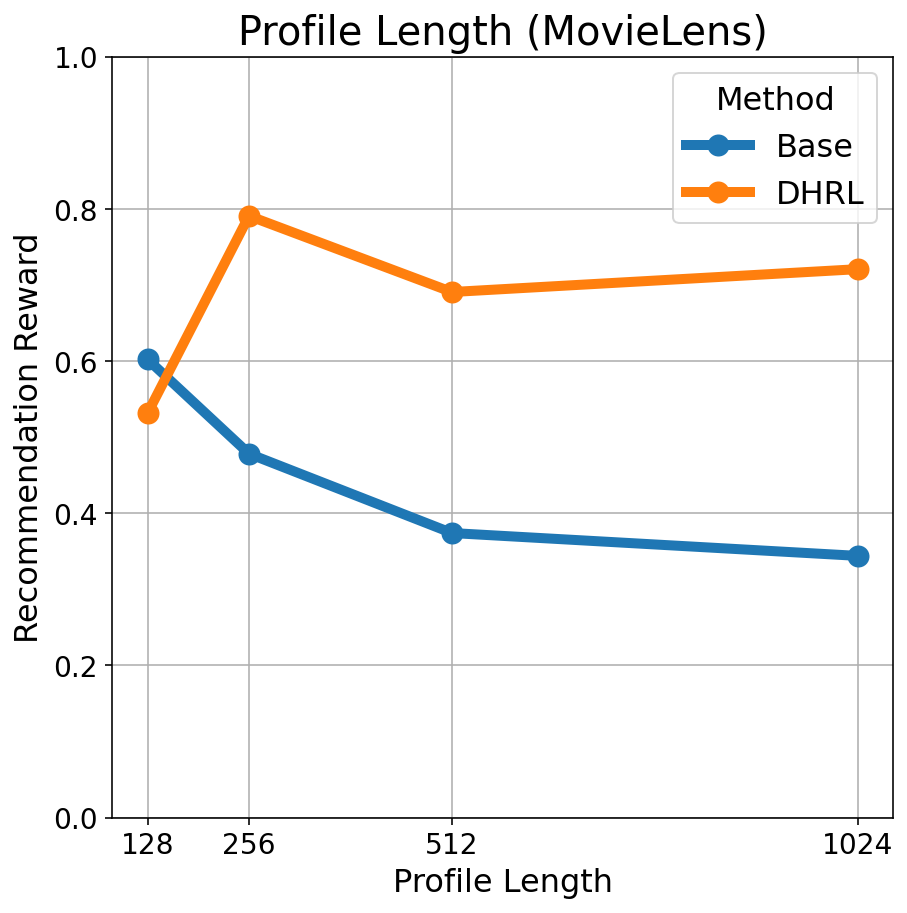}
    \includegraphics[width=0.32\linewidth]{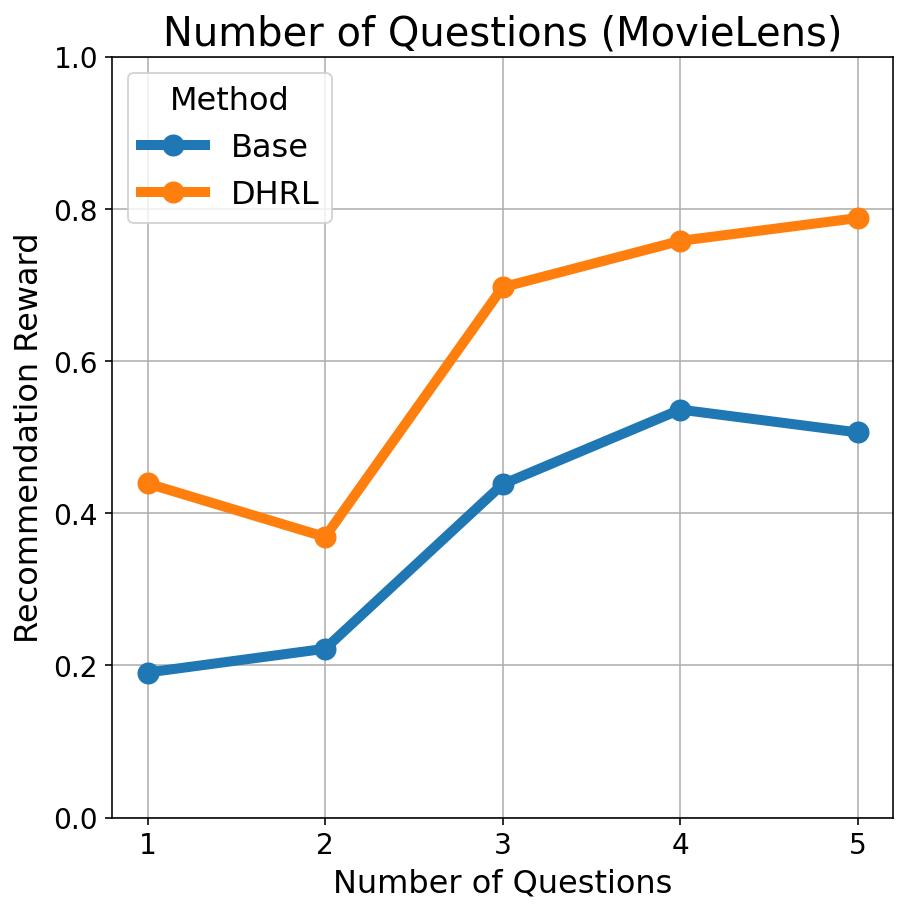}
    \caption{\footnotesize Ablation study on the history length (left), profile length (middle) and number of questions (right) on the MovieLens dataset, showing recommendation reward. DHRL (orange) consistently outperforms the baseline (blue) in all setups. We find that histories of short length (e.g., 5-10 interactions) are often sufficient for learning DHRs in our benchmarks. There is a tradeoff in selecting a profile length to achieve optimal performance, with an intermediate value (e.g., 256 tokens) often achieving best performance. Finally, the ablation study on the number of  questions shows recommendation quality generally improves as the number of questions increases.}
    \label{fig:ablations} 
\end{figure*}

\subsection{Results}

\Cref{tab:main_results} summarizes the performance of DHRL under our default configuration across models and datasets (confidence intervals provided in \Cref{appendix: misssing results}). We see that incorporating DHRL leads to substantial improvements in predictive accuracy for ranking tasks. Furthermore, AI and human evaluators rated the generated profiles favorably w.r.t.\ consistency with user history and preference predictive ability, suggesting that textual DHRs  encapsulate relevant user information effectively. We find human and AI evaluations results to be highly correlated. The notably higher recommendation rewards attained by DHRL-enhanced models confirm that DHR representations are beneficial for downstream decision-making. For Amazon Reviews, review quality metrics also indicate a strong preference for reviews conditioned on DHR profiles vs.\ baseline-generated reviews, highlighting the improved coherence and user-alignment of the generated text.

\textbf{History Length.}
The impact of history length on recommendation reward is shown in \Cref{fig:ablations} (left). DHRL consistently outperforms the baseline across all evaluated history lengths. Notably, for DHRL, optimal performance is observed with shorter interaction histories, typically around 5 to 10 interactions in our MovieLens benchmark. Beyond this range, increasing the history length tends to result in a decline in recommendation reward for DHRL, although it remains superior to the baseline. This suggests that relatively concise histories suffice for learning effective DHRs in our benchmark tasks,
and excessively long histories may yield insignificant marginal gains (and could introduce noise and complexity).

\textbf{Profile Length.}
\Cref{fig:ablations} (middle) shows the impact of DHR profile length on recommendation reward.
Performance generally improves as the profile token limit increases,
allowing profiles to encode more detailed information needed for accurate question answering and successful recommendation. However, this trend does not continue indefinitely; beyond an optimal length (e.g., around 256 tokens for MovieLens), gains diminish, with long profiles leading to a plateau or even a slight decrease in performance. This might stem from challenges in exploiting an overly large representation space, or increased difficulty for the answer agent. Our default profile length strikes a favorable balance between expressiveness and conciseness.

\textbf{Number of Questions.} 
\Cref{fig:ablations} (right) shows that recommendation reward generally increases with $K$, the number of guiding questions. This emphasizes the importance of using a sufficiently diverse set of questions to guide the representation learning process. While a few questions might provide enough signal for basic preference prediction, a richer set guides the DHR encoder $E$ towards a more robust representation of user preferences, which is crucial for maximizing downstream task performance as it likely enhances generalization.

\textbf{$f$-Divergence.}
The choice of $f$-divergence significantly impacts performance, as detailed in \Cref{tab:divergence}. Among the tested divergences, TV-distance yields the best results in both prediction accuracy and recommendation reward. This highlights that the specific mechanism used to align the answer agent's output distribution with the target distribution is a critical factor in DHRL.

\begin{table}[t!]
\centering
\caption{Ablation: Choice of f-Divergence (Amazon Dataset). Metrics after 1k iterations.}
\label{tab:divergence}
\begin{tabular}{lcc}
\toprule
Divergence & \makecell{Prediction \\ Accuracy} & \makecell{Recommendation \\ Reward} \\
\midrule
TV-distance (Default) & 0.71 & 0.83 \\
$\chi^2$-divergence & 0.42 &  0.61 \\
KL-divergence & 0.4 &  0.58 \\
\bottomrule
\end{tabular}
\end{table}

\section{Related Work}

Learning history representations is crucial for POMDPs \citep{kaelbling1998planning}. PSRs \citep{littman2001predictive} use predictions of future action-observation sequences (``tests''), though as discussed above, rely on manually engineered, often low-level tests. In contrast, DHRs learn representations based on answering high-level, semantically meaningful questions about history and future, shifting the focus
to task-relevant understanding.

Predictive principles have been adapted to deep learning methodologies. Predictive-state decoders \citep{venkatraman2017predictive}, PSRNNs \citep{downey2017predictive}, and transformer-based state predictive pepresentations \citep{liu2024enhancing} predict future observations, features, or latent states. While leveraging prediction, they typically focus on low-level features, in contrast to DHRs, which use a broader, more abstract set of questions (the QA-space) to shape representations, ensuring they answer informative queries beyond standard next-observation prediction.
Recent work has emphasized self-prediction in latent space \citep{ni2024bridging,schlegel2018discovery,schwarzer2020data,guo2020bootstrap}, learning representations by predicting future latent states, values, or observation embeddings. These methods yield latent representations implicitly defined by the prediction objective. DHRs, via the QA framework, explicitly define representation content through their questions, allowing for greater interpretability and control (see also \Cref{appendix: qa generator} on training an adversarial QA generator). DHR sufficiency is tied to question-answering, not just predicting future latents/values.

Other approaches approximate belief states, like the Wasserstein Belief Updater \citep{avalos2023wasserstein}, and often need access to the true state during training. DHRs learn from observation-action history without privileged state access. Methods like HELM \citep{paischer2023semantic} use frozen pre-trained LLMs for history compression. DHRs use a multi-agent framework to learn the representation encoder jointly with answer and decision modules, guided by reward maximization and question-answering accuracy.
Finally, LLMs have been used to generate textual user profiles \citep{hou2024large,zhang2023recommendation,zheng2024harnessing,sabouri2025towards} or personas \citep{ge2024scaling,shi2025personax} to enhance interpretability. DHRs differ as they are directly optimized to serve as sufficient statistics, with information encoded not just by generative summarization, but by their ability to answer task-relevant predictive questions.

\section{Conclusion}

This paper introduced descriptive history representations (DHRs), a novel framework for learning compact, informative history summaries by focusing on their ability to answer task-relevant questions. Our multi-agent learning approach successfully generates interpretable textual user profiles that act as effective DHRs, demonstrating strong predictive accuracy and downstream recommendation performance. DHRs offer a compelling alternative to traditional methods by explicitly guiding representation learning through a QA-space, enhancing interpretability and alignment with high-level goals. Future work includes: dynamically learning an adversarial QA generator (\Cref{appendix: qa generator}); investigating more complex question structures; and applying DHRs to a wider array of partially observable domains like robotics and dialogue systems.

\bibliography{bibliography}
\bibliographystyle{plainnat}

\newpage
\begin{appendix}

\section{Societal Impact Statement}

Descriptive history representations (DHRs), particularly when viewed as interpretable textual user profiles, offer specific societal benefits but also pose certain risks. We believe the benefits outweigh the risks if specific risk-mitigation strategies are employed in the use of DHRs. On the positive side, this approach can significantly enhance the transparency of user models in partially observable environments like recommender systems. By generating textual summaries that explain the basis for decisions, users might gain a clearer understanding of why certain recommendations are made, fostering trust and potentially allowing for more informed interactions. Moreover, DHRs can support a greater degree of user control and agency. The QA-space itself, guiding the DHR to focus on task-relevant information, could lead to more nuanced and genuinely helpful personalization, as the system learns to explicitly answer questions about user preferences and future behavior rather than just correlations with low-level signals. This explicit focus on question-answering for representation learning offers a pathway to models whose internal reasoning is more transparent.

However, the very interpretability and specificity of DHR-generated textual profiles introduce distinct challenges. Such concentrated, human-readable summaries of user history, even if aimed at better decision-making, could inadvertently reveal sensitive information or enable more precise inference of private attributes if not carefully managed, posing additional privacy risks beyond those of opaque embedding-based profiles. The content and biases of these textual profiles are also directly influenced by the choice of questions in the QA-space and the underlying language models used for generation---an improperly designed QA-space or biased LLM could lead to profiles that amplify societal biases or misrepresent users in an understandable but harmful way. Furthermore, while textual profiles aim for clarity, they are still summaries, and the nuances captured or missed are dictated by the QA-space, potentially leading to over-simplification or misinterpretations that could be exploited if the system is designed to generate persuasive, profile-aligned outputs. Responsible development necessitates careful design of the QA-space to ensure it elicits beneficial and fair representations, alongside ongoing scrutiny of the generated textual profiles for unintended consequences.

\newpage
\section{QA Generator (Design and Adversarial Training)}
\label{appendix: qa generator}

The descriptive history representation (DHR) framework relies on a question-answer space (QA-space) to guide the learning of informative representations. A crucial component is the QA generator ($\nu^*_{QA}$), which provides question-answer pairs $(q,y)$ based on the history $h$ and future outcomes $\omega$. In the main paper, we largely assume access to such a generator. This section elaborates on how this was handled in our experiments, discusses the design of sufficient question sets, and proposes a future direction for learning useful QA generators.

\subsection{Experiment Design}
\label{appendix: QA generator experiment design}

Our work primary focuses on learning DHRs and the properties of the resulting representations, rather than on learning the QA generator itself. Consequently, for our experiments, we assume access to a pre-defined QA generator. This is implemented by procedurally generating questions and their ground-truth answers based on observed user histories and future interactions from the datasets. 

Our questions are aimed at capturing salient aspects of user preferences and future behavior, rather than predicting raw, low-level observations. In our experiments, ranking questions include comparing pairs of items based on predicted user preference, such as ``Does the user prefer Item A or Item B?'' The ground-truth answers for these is derived from the user's future ratings of these items. 

For the Amazon dataset, we also include review generation questions, asking to generate a textual review for an item as the user might write it; e.g., ``Write a review this user might provide the following product: <description>." The ground-truth answer here is the actual review written by the user for that item in their future interactions. 

Our questions rely on natural language templates, whose slots are filled with specific items or contexts from the user's future. This approach ensures that questions are not overly specific to individual data points but rather represent generalizable queries about user behavior and preferences, leveraging the richness and flexibility of natural language. This methodology moves beyond attempting to predict the exact sequence of raw future observations (e.g., raw ratings or clicks). Instead, it focuses the representation learning on summarizing history in a way that supports answering these more abstract, semantically meaningful questions, which are directly relevant to the decision-making task, like making good recommendations. Careful design of these questions is crucial for ensuring the DHR learns to encapsulate information relevant to understanding user preferences and predicting their future actions in a structured manner.

\subsection{Designing Sufficient Question Sets}
The choice of questions fundamentally defines the information that the DHR aims to capture. A set of questions $Q^*$ is deemed sufficient if the answers to these questions provide all necessary information for a specific purpose, such as maximizing reward in an RL task. 

Designing a set of sufficient questions is analogous to other critical design choices in RL, such as defining an action space (or a set of options), an abstraction of the state space, or crafting a reward function. In many RL applications, these components are manually engineered based on domain knowledge and the specific goals of the agent, and are often refined iteratively. For instance, in robotics, actions might be joint torques, and rewards might be based on task completion. Similarly, a set of questions for DHRs can be designed by domain experts to probe the most relevant aspects of the environment's state or history. 

A heuristic approach for creating a sufficient question set is to iteratively refine or expand the set of questions based on the agent's performance or by analyzing what information seems to be missing from the DHR. For example, in a dialogue system, initial questions might focus on user intent or sentiment, with later additions querying user knowledge or specific entities. In autonomous driving, questions could range from simple presence detection like ``Is there a pedestrian in the crosswalk?" to more complex predictions like "What is the predicted trajectory of the vehicle ahead?" 

A well-designed set of questions can be highly beneficial. Questions, especially in natural language, make the DHR's content more interpretable, as we can understand what information the representation is trying to encode. It also explicitly directs the representation learning process towards capturing information deemed critical by the designer. Furthermore, compared to learning representations based on predicting all raw future observations---as seen in some PSR or belief state approaches---focusing on a curated set of high-level questions can be more tractable and generalizable, particularly in complex environments with high-dimensional observation spaces. 

While learning the questions themselves is a desirable long-term goal, designing the question set using the strong prior of large language models (LLMs) is a practical and powerful approach, which injects domain knowledge and task relevance into the representation learning process.
This practicality and power stem from the vast world knowledge and semantic understanding embedded within LLMs, acquired during their pre-training on diverse and extensive corpora.
Consequently, an LLM can be prompted to generate candidate questions with domain-specific nuances---for example, recognizing relevant attributes for movies, such as genre, director, or thematic elements; or for products, such as brand, material, or user reviews.

Beyond domain specifics, LLMs possess a strong grasp of task-relevant concepts. For a recommendation task, this includes notions of preference, comparison, a user's potential future behavior, or even the underlying reasons for a choice, enabling the generation of questions that directly probe these critical aspects.
This allows the LLM-guided design process to formulate questions that elicit more abstract and semantically rich answers than merely predicting low-level future observations.
For instance, rather than focusing solely on predicting the next click, questions can probe comparative preferences (e.g., ``Would the user prefer item A over item B given their history?'') or solicit generative summaries of latent preferences (e.g., ``Describe the user's taste profile based on past interactions.'').

The process might involve leveraging an LLM to generate a broad suite of potential questions tailored to the specific domain and task, from which a human designer can then curate, refine, or further specialize the set, significantly augmenting the human designer's capacity to craft a comprehensive and effective question set.

By injecting such structured and high-level priors into the question design phase, the DHR learning process is more directly guided towards capturing the most salient information for effective decision-making and interpretability, aligning the representation with the core objectives of the task.

\subsection{Training a QA Agent}

Our work focuses on the core DHR learning algorithm with a fixed QA-space. A promising future direction is to learn the QA generator itself. One approach is to formulate this as an adversarial learning problem, where a QA agent, or generator, learns to pose questions and provide answers that are maximally informative or challenging for the DHR encoder and answer agent. 

The original optimization problem (\Cref{eq: DHR primal objective}) is designed to to maximize task reward while minimizing the divergence between the DHR's predicted answer distribution and a fixed, ground-truth answer distribution. If we aim to learn the QA generator, specifically the component that provides these ``ground-truth'' or ``target'' answers (and potentially the questions themselves), we can introduce an adversarial objective. 

We refine the objective in \Cref{eq: DHR primal objective} as follows:
\begin{align}
\label{eq: primal adversarial QA}
    \max_{E, \nu_A, \pi_D} 
    \min_{\nu^*_{QA}}
        (1-\lambda) V(\pi) 
        - 
        \lambda 
        \fdiv{d^{\nu^*_A}}{d^{\nu_A}},
\tag{OPT 3}
\end{align}
where the adversarial QA generator $\nu^*_{QA}$ (which generates $(q,y)$ pairs from $h, \omega$) tries to select questions and answers to maximize the divergence $D_f$, making it harder for the DHR $E$ and answer agent $\nu_A$ to match its outputs. The DHR encoder $E$ and answer agent $\nu_A$ continue to try to minimize this divergence while maximizing task reward.

The dual variational form, analogous to \Cref{eq: dual improved}, becomes
\begin{align}
    \label{eq: dual qa generator}
    \max_{\pi_D, E, \nu_A}
    \min_{\substack{g: \gH \times \gQ \times \gY \mapsto \sR \\ \nu_{QA}^*: \gH \times \Omega \mapsto \Delta_{\gQ \times \gY}}} 
    \expect*{}{
        (1 - \lambda) r(h, a) 
        +
        \lambda
        \expect*{\substack{q \sim \nu^*_Q(h, \omega) \\ y \sim \nu_A(q,h,\omega)}}{f^*(g(h,q,y))}
        - 
        \lambda 
        \expect*{\substack{q \sim \nu^*_Q(h, \omega) \\ y \sim \nu^*_{A}(z, q)}}{g(h,q,y)}
    }.
    \tag{OPT 4}
\end{align}
In this formulation, $\nu^*_{QA}$ is the adversarial QA generator that produces $(q,y)$ pairs given history $h$ and future $\omega$. The agent's answer network $\nu_A$ predicts $\hat{y}$ for a question $q$ (generated by $\nu^*_{QA}$) using the DHR $z=E(h)$. The discriminator $g$ tries to distinguish between answers from $\nu_A$ and $\nu^*_{QA}$. The adversarial QA generator $\nu^*_{QA}$ is trained to generate $(q,y)$ pairs that are hard for $\nu_A$ to predict correctly, effectively maximizing the objective for $g$. \Cref{eq: dual qa generator} maintains strong duality, allowing for an iterative training algorithm involving updates to $\pi_D, E, \nu_A, g,$ and now also $\nu^*_{QA}$.

Learning an adversarial QA generator offers potential benefits such as the automated discovery of informative questions, where the system could learn to ask questions most relevant for the task or that expose gaps in the DHR's understanding. It could also facilitate a form of curriculum learning, where the adversarial QA generator might initially pose simpler questions and gradually increase their difficulty as the DHR encoder and answer agent improve. 

Despite the benefits above, ensuring that the learned questions are semantically meaningful and interpretable may be challenging, as without proper regularization or priors, the generator might find trivial or uninformative ways to make the task hard. LLMs serve as powerful backbones for $\nu^*_{QA}$, leveraging their inherent understanding of language and concepts to propose coherent questions. Still, this would require strong regularization, e.g., by encouraging question diversity, penalizing overly complex questions, or biasing towards human-understandable questions, to guide the generator effectively and prevent collapse or trivial solutions. 

Learning an adversarial QA generator is a complex but exciting research direction, which may significantly enhance the adaptability of the DHR framework. Our current work lays the foundation by demonstrating the effectiveness of DHRs given a well-defined QA-space, and future work can build upon this to explore dynamic and learned QA-spaces.

\newpage
\section{Implementation Details}
\label{appendix: implementation details}

This appendix provides additional details regarding the implementation of our DHR learning (DHRL) framework.

\paragraph{Models and Parameterization}
The DHR encoder ($E$), answer agent ($\nu_A$), and decision agent ($\pi_D$) are implemented using LLMs. As mentioned in the main paper (Section 5.1), primary LLM models include Gemini 1.5 Flash and Gemma V3 (4B, 12B). 
The DHR encoder and answer agent share the same base LLM architecture. A fixed anchor LLM (typically the original or supervised fine-tuned version of the model), is used to provide the reference distribution for the KL-divergence regularization term in the policy loss.
The value network, responsible for estimating advantages and providing outputs for the DSR discriminator ($g$), defaults to a Gemma V3 4B model unless specified otherwise.
The discriminator $g$ outputs are derived from the value network. This is achieved by feeding the value network QA pair samples (one reference, one target, as required by DSR). The value network's output vocabulary is conceptually partitioned: one segment is used for standard value prediction, while another segment provides the scalar outputs leveraged by the discriminator logic.

\paragraph{Training Framework and Algorithm}
The core learning algorithm is implemented as an actor-critic method. The DHR encoder acts as the policy, and the answer agent components (value function and discriminator logic) form parts of the critic and learning signal.
Training is run for up to 3000 optimization steps. As noted in the main paper, convergence was generally observed within 1000 iterations, which forms the basis for our reported results.

\paragraph{Key Hyperparameters}
Several hyperparameters governed the training process. These are summarized in Table~\ref{tab:hyperparameters_appendix_c}.

\begin{table}[h!]
\centering
\caption{Key Hyperparameters for DHRL Experiments.}
\label{tab:hyperparameters_appendix_c}
\begin{tabular}{@{}ll@{}}
\toprule
Hyperparameter                     & Default Value  \\ \midrule
DSR Balancing Factor ($\lambda$)   & 0.01  \\
Batch Size                         & 128       \\
LLM Input Token Length             & 6144     \\
User Profile ($z$) Max Token Length & 256 (default)   \\
User History Length ($H$)          & 10 (default)        \\
Number of Ranking Questions ($K_q$) & 5 (default)          \\   
Policy Update Delay                & 20 steps     \\
Policy Warmup Steps                & 20 steps    \\
 \bottomrule
\end{tabular}
\end{table}

\paragraph{Loss Function Details (DSR)}
The choice of $f$-divergence for the discriminator was crucial. The default was TV-distance, as reported in the main paper. KL divergence and Chi-squared divergence were also explored. 
The DHR encoder's policy was regularized using KL divergence w.r.t.\ the fixed anchor (SFT) model, with the KL weight ($\alpha$) annealing. The total loss driving updates included: a policy gradient loss term derived from advantages; a value function loss term; the KL regularization term for the policy; and the discriminator loss term. The advantages were computed based on a sum of the environment reward $r^D$ (related to downstream task performance) and the DSR answer reward $r^A$.

\newpage
\section{Missing Results}
\label{appendix: misssing results}

\Cref{tab:main_results_a_appendix,tab:main_results_b_appendix} provide the complete table of results from \Cref{tab:main_results} with 95\% confidence intervals. We also provide missing results for ablations for Amazon in \Cref{fig:amazon ablations}, showing similar results to \Cref{fig:ablations}. We also show in \Cref{fig:amazon ablations} (left) the learning curves for the ablation over profile lengths. We found that the Amazon dataset in particular is very sensitive to longer profile lengths. Future work can further explore the effects of the DHR length and its relation to the choice of DHR templates with their predictive capabilities.

\begin{table}[t!]
\definecolor{lightgray}{gray}{0.9} 
\centering
\caption{\footnotesize Performance (with 95\% CI) of DHRL with default settings (profile with max. 256 tokens, 5 questions, 10 history interactions, TV-divergence) after 1,000 iterations for Gemma V3 4B, 12B, and Gemini 1.5 Flash models.}
\label{tab:main_results_a_appendix}
\begin{tabular}{lcc}
\toprule
\bf Model \& Method & \bf \makecell{Prediction \\ Accuracy \\ (w.r.t. GT)} & \bf \makecell{Rec. \\ Reward \\ ($\boldsymbol{r^D}$)} \\
\midrule\midrule
\multicolumn{3}{c}{Task: recommendation using Amazon products user profiles} \\
\midrule
Gemma V3 4B & 0.34 $\pm$ 0.04  & 0.54 $\pm$ 0.04  \\
\rowcolor{lightgray} Gemma V3 4B+DHRL  & 0.71 $\pm$ 0.04 & 0.83 $\pm$ 0.03 \\
Gemma V3 12B & 0.67 $\pm$ 0.04  & 0.78 $\pm$ 0.03 \\
\rowcolor{lightgray} Gemma V3 12B+DHRL  & 0.75 $\pm$ 0.04 & 0.84 $\pm$ 0.03 \\
Gemini 1.5 Flash & 0.69 $\pm$ 0.04  & 0.80 $\pm$ 0.03 \\
\rowcolor{lightgray} Gemini 1.5\ Flash+DHRL & 0.74 $\pm$ 0.04  & 0.86 $\pm$ 0.03 \\
\midrule\midrule
\multicolumn{3}{c}{Task: recommendation using MovieLens user profiles} \\
\midrule
Gemma V3 4B & 0.37 $\pm$ 0.04  & 0.61 $\pm$ 0.04 \\
\rowcolor{lightgray} Gemma V3 4B+DHRL & 0.74 $\pm$ 0.04 & 0.79 $\pm$ 0.03 \\
Gemma V3 12B & 0.58 $\pm$ 0.04  & 0.64 $\pm$ 0.04 \\
\rowcolor{lightgray} Gemma V3 12B+DHRL  & 0.84 $\pm$ 0.03 & 0.93 $\pm$ 0.02 \\
Gemini 1.5 Flash & 0.62 $\pm$ 0.04 & 0.68 $\pm$ 0.04 \\
\rowcolor{lightgray} Gemini 1.5\ Flash+DHRL & 0.82 $\pm$ 0.03 & 0.92 $\pm$ 0.02 \\
\bottomrule
\end{tabular}
\end{table}

\begin{table}[t!]
\definecolor{lightgray}{gray}{0.9} 
\centering
\caption{\footnotesize Performance (with 95\% CI) of DHRL with default settings (profile with max. 256 tokens, 5 questions, 10 history interactions, TV-divergence) after 1,000 iterations for Gemma V3 4B, 12B, and Gemini 1.5 Flash models.}
\label{tab:main_results_b_appendix} 
\begin{tabular}{lccc}
\toprule
\bf Model \& Method &  \bf \makecell{Profile-History \\ Consistency \\ (AI/Human)} & \bf \makecell{~ Prediction \\ Fidelity \\ (AI/Human)} & \bf \makecell{Review \\ Quality \\ (AI/Human)} \\
\midrule\midrule
\multicolumn{4}{c}{Task: recommendation using Amazon products user profiles} \\
\midrule
Gemma V3 4B & 3.46 $\pm$ 0.16~/~3.28 $\pm$ 0.16 & 0.44 $\pm$ 0.04~/~0.47 $\pm$ 0.04 & 0.15 $\pm$ 0.03~/~0.27 $\pm$ 0.04\\
\rowcolor{lightgray} Gemma V3 4B+DHRL  & 3.41 $\pm$ 0.16~/~3.85 $\pm$ 0.15 & 0.56 $\pm$ 0.04~/~0.53 $\pm$ 0.04 & 0.85 $\pm$ 0.03~/~0.73 $\pm$ 0.04\\
Gemma V3 12B & 4.32 $\pm$ 0.12~/~4.18 $\pm$ 0.13 & 0.34 $\pm$ 0.04~/~0.39 $\pm$ 0.04 & 0.17 $\pm$ 0.03~/~0.35 $\pm$ 0.04\\
\rowcolor{lightgray} Gemma V3 12B+DHRL   & 4.39 $\pm$ 0.12~/~4.37 $\pm$ 0.12 & 0.66 $\pm$ 0.04~/~0.61 $\pm$ 0.04 & 0.83 $\pm$ 0.03~/~0.65 $\pm$ 0.04\\
Gemini 1.5 Flash  & 4.41 $\pm$ 0.12~/~4.30 $\pm$ 0.13 & 0.38 $\pm$ 0.04~/~0.41 $\pm$ 0.04 & 0.17 $\pm$ 0.03~/~0.38 $\pm$ 0.04\\
\rowcolor{lightgray} Gemini 1.5\ Flash+DHRL  & 4.46 $\pm$ 0.11~/~4.34 $\pm$ 0.12 & 0.62 $\pm$ 0.04~/~0.59 $\pm$ 0.04 & 0.83 $\pm$ 0.03~/~0.62 $\pm$ 0.04\\
\midrule\midrule
\multicolumn{4}{c}{Task: recommendation using MovieLens user profiles} \\
\midrule
Gemma V3 4B & 3.72 $\pm$ 0.15~/~3.96 $\pm$ 0.14 & 0.32 $\pm$ 0.04~/~0.45 $\pm$ 0.04 & -\\
\rowcolor{lightgray} Gemma V3 4B+DHRL & 3.84 $\pm$ 0.15~/~4.24 $\pm$ 0.13 & 0.68 $\pm$ 0.04~/~0.55 $\pm$ 0.04 & -\\
Gemma V3 12B & 4.66 $\pm$ 0.09~/~4.58 $\pm$ 0.10 & 0.35 $\pm$ 0.04~/~0.42 $\pm$ 0.04 & -\\
\rowcolor{lightgray} Gemma V3 12B+DHRL  & 4.62 $\pm$ 0.10~/~4.71 $\pm$ 0.09 & 0.76 $\pm$ 0.04~/~0.66 $\pm$ 0.04 & -\\
Gemini 1.5 Flash & 4.69 $\pm$ 0.09~/~4.53 $\pm$ 0.11 & 0.38 $\pm$ 0.04~/~0.39 $\pm$ 0.04 & -\\
\rowcolor{lightgray} Gemini 1.5\ Flash+DHRL  & 4.74 $\pm$ 0.08~/~4.70 $\pm$ 0.09 & 0.77 $\pm$ 0.03~/~0.62 $\pm$ 0.04 & -\\
\bottomrule
\end{tabular}
\end{table}

\begin{figure*}[t!]
    \centering
    \includegraphics[width=0.375\linewidth]{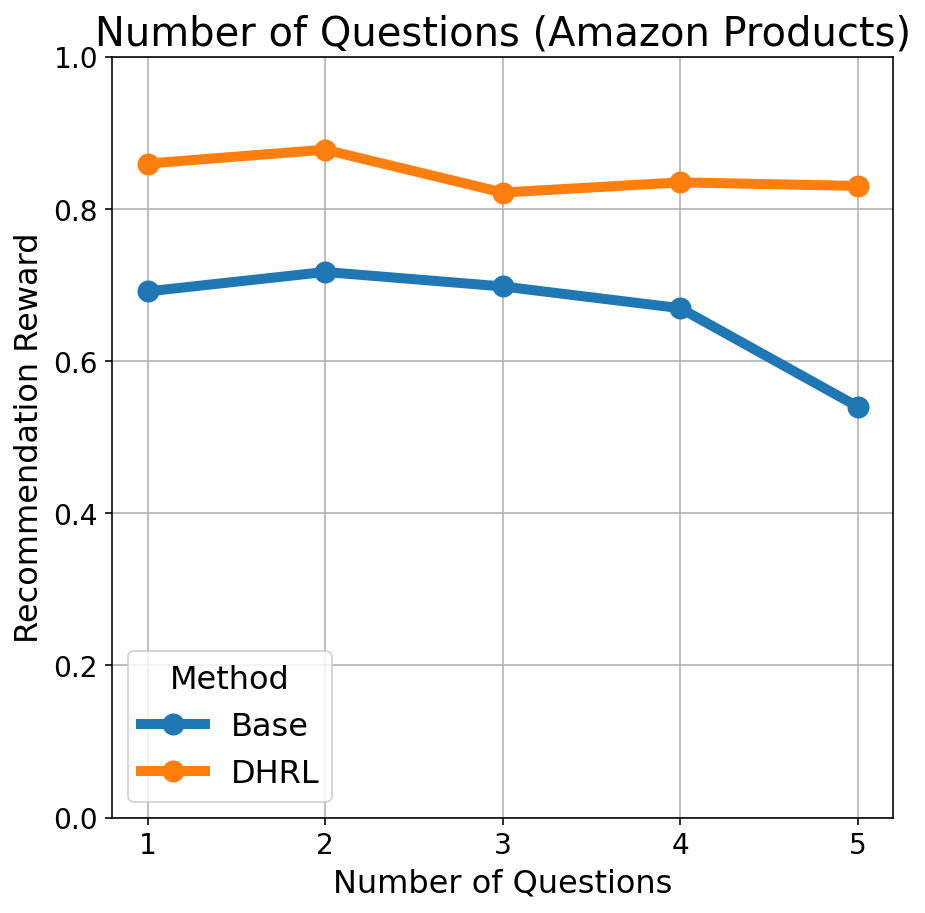}
    \includegraphics[width=0.4\linewidth]{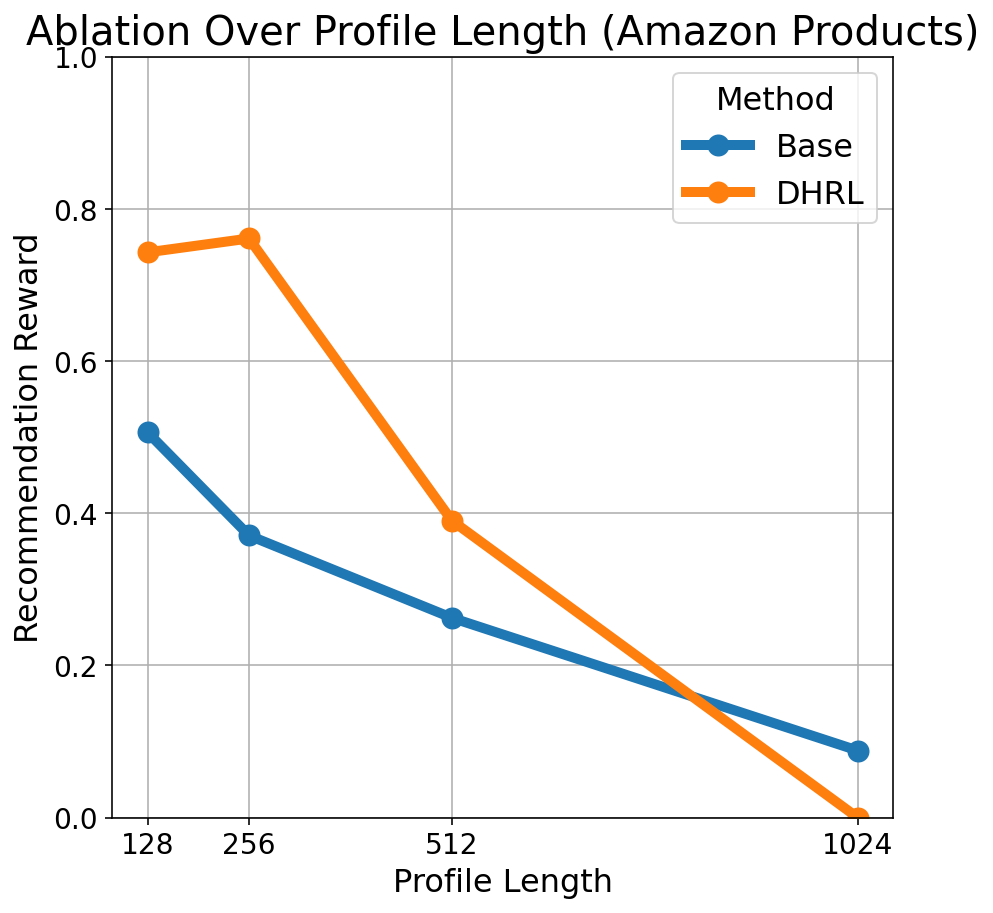}
    \includegraphics[width=0.8\linewidth]{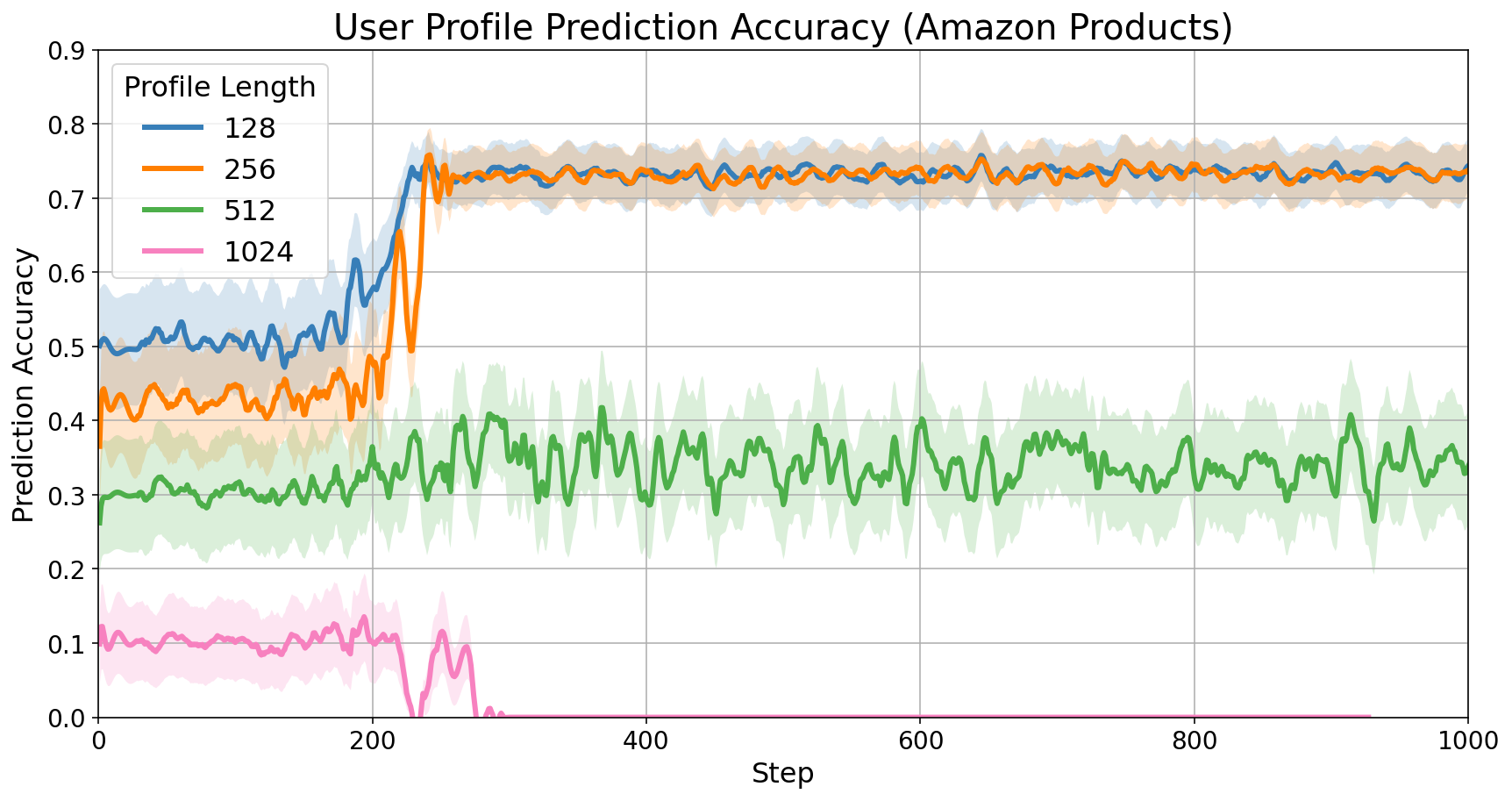}
    \caption{\footnotesize Ablation study on the history length (left) and profile length (right) the Amazon dataset, showing recommendation reward. Bottom plot shows learning curves for different profile lengths. Lower profile lengths were critical to ensure convergence.}
    \label{fig:amazon ablations} 
\end{figure*}

\newpage
\section{Proof of \Cref{proposition: DHR is sufficient}}
\label{appendix: proof of dhr proposition}

\dhrsufficientproposition*

\begin{proof}
Let $E_{DHR}: \mathcal{H} \rightarrow \mathcal{Z}$ be a Descriptive History Representation.
By the definition of an $f$-sufficient QA-space, for any history $h \in \gH$, there is a set of sufficient questions $Q^*_h \subseteq \gQ$. Denote the set of sufficient question-answer pairs 
\begin{align*}
S_{QA}(h) := \{(q, \nu(h,q)) \mid q \in Q^*_h\}.
\end{align*}
By \Cref{def: DHR}, for any $h \in \mathcal{H}$ and any sufficient question $q \in Q^*_h$:
\begin{align*}
\nu_A(E_{DHR}(h), q) = \nu(h,q).
\end{align*}
This implies we can write $S_{QA}(h) = \brk[c]*{ (q, \nu_A(E_{DHR}(h),q)) \mid q \in Q^*_h }$.

By the definition of a sufficient QA space, there exists a function $F^\pi: (\gQ \times \gY)^* \rightarrow \mathbb{R}$ such that 
\begin{align}
\label{eq: sufficient qa space}
f^\pi(h) = F^\pi(S_{QA}(h)).
\end{align}

To show $E_{DHR}$ is an $f$-sufficient statistic, we must define $g^\pi : \mathcal{Z} \rightarrow \mathbb{R}$ such that 
\begin{align}
\label{eq: to show}
f^\pi(h) = g^\pi(E_{DHR}(h)).
\end{align}

Let $z \in \mathcal{Z}$. Define the set $S_{constr}(z)$ of QA pairs constructible from $z$ as:
\begin{align*}
S_{constr}(z) := \brk[c]*{ (q, \nu_A(z,q)) \mid \exists h' \text{ s.t. } z = E_{DHR}(h') \land q \in Q^*_{h'} }
\end{align*}
The set $S_{constr}(z)$ includes all question-answer pairs where the question $q$ is sufficient for some history $h'$ that maps to $z$, and the answer is generated by $\nu_A$ using $z$.
Define $g^\pi(z) := F^\pi(S_{constr}(z))$.

Now, consider a specific history $h$, and let $z_h = E_{DHR}(h)$. We need to show $g^\pi(z_h) = f^\pi(h)$. By definition, we have that

\begin{align}
\label{eq: connecting g and F}
g^\pi(E_{DHR}(h)) = g^\pi(z_h) = F^\pi(S_{constr}(z_h)).
\end{align}

By \Cref{eq: sufficient qa space}, $f^\pi(h) = F^\pi(S_{QA}(h))$. Thus, to prove \Cref{eq: to show}, given the equality in \Cref{eq: connecting g and F}, it is left to show $F^\pi(S_{constr}(z_h)) = F^\pi(S_{QA}(h))$.

To establish $F^\pi(S_{constr}(z_h)) = F^\pi(S_{QA}(h))$, we will show that $S_{constr}(z_h) = S_{QA}(h)$.

Denote the question sets corresponding to $S_{constr}(z_h)$ and $S_{QA}(h)$ by
\begin{align*}
K_{QA} = Q^*_h, \qquad
K_{constr} = \brk[c]*{ q \mid \exists h' \text{ s.t. } z_h = E_{DHR}(h') \land q \in Q^*_{h'} }
\end{align*}
To show that $S_{constr}(z_h) = S_{QA}(h)$, it is enough to show $K_{QA} = K_{constr}$ (since $E_{DHR}$ is a DHR and thus by \Cref{def: DHR} the answer functions are equivalent, i.e., $\nu(h,q) = \nu_A(E_{DHR}(h), q)$).

\textbf{Case 1:  $K_{QA} \subseteq K_{constr}.$}
    Let $q_0 \in K_{QA}$. Then $q_0 \in Q^*_h$.
    We can choose $h' = h$. Since $z_h = E_{DHR}(h)$ by definition of $z_h$, the condition $z_h = E_{DHR}(h')$ is satisfied.
    Also, $q_0 \in Q^*_{h'}$ because $h'=h$.
    Thus, $q_0 \in K_{constr}$. So, $K_{QA} \subseteq K_{constr}$.

\textbf{Case 2:  $K_{constr} \subseteq K_{QA}.$}
    Let $q_1 \in K_{constr}$. Then there exists some history $h_s$ such that $z_h = E_{DHR}(h_s)$ and $q_1 \in Q^*_{h_s}$.
    We need to show that $q_1 \in Q^*_h$.
    Notice that, if multiple histories $h^1, h^2, \dots$ map to the same representation $z = E_{DHR}(h^1) = E_{DHR}(h^2) = \dots$, then the set of sufficient questions is in itself a function of $z$. That is, $Q^*_h = G(E_{DHR}(h))$ for some function $G:\gZ \mapsto \gQ$ mapping representations to sets of questions.
    If this holds, then for our specific history $h$, $Q^*_h = G(E_{DHR}(h)) = G(z_h)$. So $K_{QA} = G(z_h)$.
    For $q_1 \in K_{constr}$, we have $E_{DHR}(h_s) = z_h$ and $q_1 \in Q^*_{h_s}$. Under the interpretation $Q^*_{h_s} = G(E_{DHR}(h_s))$, we have $q_1 \in G(z_h)$.
    Thus, $q_1 \in K_{QA}$. This shows $K_{constr} \subseteq K_{QA}$.

Since $K_{QA} \subseteq K_{constr}$ and $K_{constr} \subseteq K_{QA}$, we have $K_{QA} = K_{constr}$. This completes our proof.
\end{proof}

\newpage
\section{LLM Prompts}
\label{appendix: agent prompts}

Below we provide the complete prompts used in DHRL. The prompts are provided for the Amazon domain. The prompt used for Movielens is identical except for removal the review question.

Text in {\color{blue}blue} refers to parts of the prompt that are not changed. Text in {\color{red}red} refers to placeholders that are replaced with e.g., titles, descriptions, user profiles, etc.

\subsection{DHR Encoder}

{
\color{blue}
Your task is to write a user profile in the shopping domain for Clothing, Shoes, and Jewelry. The user profile should describe the user's preferences well. \\
Your user profile will later be used to predict the user's preferences and reviews over new items.\\
--------------------------------------------------\\
You will be given a list of products the user has rated (ratings 1 to 5, where 5 is the highest rating).\\
You will also be given reviews for some of these products.\\
You will then be asked to write a user profile describing the preferences of the user.\\
--------------------------------------------------\\
--------------------------------------------------\\
Here is an example of a user profile given a history:\\
\#BEGIN EXAMPLE
~\\~\\        
  User History: \\
    Item \#1 \\
    Item title: Robert Graham Men's Shipwreck Long Sleeve Button Down Shirt \\
    Item description: N/A \\
    Item price: N/A \\
    User rating: 5.0 \\
    User review: Fantastic shirt.  Great price.  I saw the same shirt at Nordstrom for 229.00
 ~\\~\\   
--------------------------------------------------
~\\~\\
    Item \#2 \\
    Item title: Robert Graham Men's Whitehorse Long Sleeve Button Down Shirt \\
    Item description: "The geometric pattern on this dress shirt creates an optical illusion to enhance your wardrobe. Crafted from Egyptian cotton for comfort.", "The geometric pattern on this dress shirt creates an optical illusion to enhance your wardrobe. Crafted from Egyptian cotton for comfort." \\
    Item price: N/A \\
    User rating: 5.0 \\
    User review: Thank you for such a great deal on a great shirt
  ~\\~\\  
--------------------------------------------------
~\\~\\
    Item \#3 \\
    Item title: Robert Graham Men's Shawn-Long Sleeve Woven Shirt \\
    Item description: "When is a polo not just a polo? When it's made from exceptional mercerized pique cotton and boasting our signature space dyed tipping at the collar to take it from ordinary to extraordinary." \\
    Item price: N/A \\
    User rating: 5.0 \\
    User review: I have never had so many complements on a shirt. \\
Thank you for the great deal
  ~\\~\\  
--------------------------------------------------
~\\~\\
    Item \#4 \\
    Item title: Robert Graham Men's Redzone Pixel Patterned Button-Front Shirt with Convertible Cuffs \\
    Item description: N/A \\
    Item price: N/A \\
    User rating: 5.0 \\
    User review: love it
   ~\\~\\
--------------------------------------------------
~\\~\\
    Item \#5 \\
    Item title: Robert Graham Men's Shipwreck-Long Sleeve Button Down Shirt \\
    Item description: "Every stylish guy needs a sophisticated striped shirt in his arsenal. Crafted from premium Egyptian cotton with space dyed stripes and a pop of paisley embroidery at the cuffs and inside placket for a little extra edge, this shirt has you covered at the office or a weekend away" \\
    Item price: N/A \\
    User rating: 5.0 \\
    User review: love it
  ~\\~\\  
--------------------------------------------------
~\\~\\
    Item \#6 \\
    Item title: Robert Graham Men's Maitai-Long Sleeve Button Down Shirt \\
    Item description: "Rain check. No way. Pack a mai tai and check's in for sunny days. Ombre squares of crisp bright white, ocean blues and poolside purples along with a dash of sunny blue sky over embroidery are the perfect cocktail for any getaway" \\
    Item price: N/A \\
    User rating: 5.0 \\
    User review: love it
  ~\\~\\  
--------------------------------------------------
~\\~\\
    Item \#7 \\
    Item title: Lucky 21 Women's Plus Size Cuffed Sleeve One Button Boyfriend Blazer \\
    Item description: "This boyfriend blazer is perfect for office and work wear, but also for a night out with friends! It features one button and cuffed sleeves." \\
    Item Price: \$8.00 - \$34.99 \\
    User rating: 1.0 \\
    User review: sent back
   ~\\~\\ 
--------------------------------------------------
~\\~\\
    Item \#8 \\
    Item title: Robert Graham Men's Seattle Long Sleeve Button Down Shirt \\
    Item description: "Hombre fade window pane print with satin stitch window pane edging and floral embroidery woven. Has digital printed neck band, back yoke and under collar." \\
    Item price: N/A \\
    User rating: 5.0 \\
    User review: Great shirt.  Robert graham is the king
  ~\\~\\  
--------------------------------------------------
~\\~\\
    Item \#9 \\
    Item title: EleQueen 925 Sterling Silver Full Cubic Zirconia Bridal Flower Stud Earrings 15mm \\
    Item description: N/A \\
    Item Price: \$16.99 \\
    User rating: 5.0 \\
    User review: The wife likes
  ~\\~\\  
--------------------------------------------------
~\\~\\
    Item \#10 \\
    Item title: Leveret Women's Robe Soft Micro Fleece Plush Shawl Collar Bathrobe Robe (Size Small-XXLarge) \\
    Item description: N/A \\
    Item Price: \$39.99 - \$49.99 \\
    User rating: 5.0 \\
    User review: Wife likes
   ~\\~\\ 
   ~\\~\\ 
    User Profile:
  ~\\~\\  
    \#PROFILE\#
\begin{enumerate}
    \item  Brand Loyalty:
    \begin{itemize}
    \item   Extremely High Loyalty to Robert Graham: The vast majority (7 out of 10) of rated items are from this brand.  All Robert Graham items are rated 5/5. This is the single strongest signal in the data. The user explicitly calls Robert Graham "king".
    \item   Open to Other Brands, but with Caution:  Purchased items from Lucky 21 (blazer, rated poorly), EleQueen (earrings, positive), and Leveret (robe, positive). This suggests a willingness to explore, but success outside of Robert Graham is mixed.
    \end{itemize}
\item  Product Category Preferences:
\begin{itemize}
        \item   Men's Shirts (Dominant):  Specifically, long-sleeve button-down shirts are the primary interest.  Focus is on dress shirts and slightly more casual shirts (like a refined polo).
        \item   Women's Gifts/Items (Secondary): Purchases include a women's blazer, earrings, and a robe.  Reviews indicate these are likely gifts ("Wife likes").  The low rating on the blazer suggests a different preference set for women's clothing, or simply a bad experience with that specific item.
        \item    No interest shown in shoes, and lower interest in jewerly.
\end{itemize}
\item  Style Preferences (Men's Shirts):
\begin{itemize}
        \item   Bold Patterns and Colors:  Descriptions of favored items include "geometric pattern," "ombre squares," "pixel patterned," "striped," "floral embroidery."  This user is not afraid of visually interesting shirts.
        \item   High-Quality Materials: Mentions of "Egyptian cotton," "mercerized pique cotton," "satin stitch," and "925 Sterling Silver" (for the earrings) indicate a preference for premium materials.
        \item   Details and Accents:  Appreciates features like "paisley embroidery," "space-dyed tipping," "convertible cuffs," and "digital printed neckband." These small details are important.
        \item Business Casual, elevated.
\end{itemize}
\item  Style Preferences (Gifts/Women's Items):
\begin{itemize}
        \item   Classic/Elegant: The earrings are described as "Bridal Flower Stud" with Cubic Zirconia, suggesting a classic, somewhat formal style. The robe is a "Soft Micro Fleece Plush Shawl Collar," indicating a preference for comfort and a traditional design.
        \item   Less information than men's.
\end{itemize}
\item  Price Sensitivity:
\begin{itemize}
        \item   Value-Oriented, but Willing to Pay for Quality:  The user repeatedly mentions "great deal" and compares the price favorably to Nordstrom's. This suggests a search for good value, but not necessarily the absolute lowest price. They are willing to pay for a Robert Graham shirt, implying a higher price bracket acceptance for favored brands.
        \item   Wide Price Range: Has purchased items ranging from under \$10 (potentially the blazer, depending on the specific price within the range) to likely over \$100 (based on the Nordstrom comparison for Robert Graham shirts).
\end{itemize}
\item  Review Style:
\begin{itemize}
        \item   Concise and Positive (when satisfied): Uses short phrases like "love it," "Fantastic shirt," "Great shirt."
        \item   Value-Focused: Often mentions price and deals in positive reviews.
        \item   Direct (when dissatisfied):  Simply states "sent back" for the negative review.
        \item   Expressive: show appreciation and gratefulness.
\end{itemize}
\item Purchase Channel:
        The purchase channel is not described.
\item Location/Region:
        No address, etc. is listed.
\item  Purchase Frequency:
        Cannot be determined from this data alone. More history would be needed.
\item. Inferred Gender and Demographic
        Most likely Male, purchasing gifts for female.
\end{enumerate}
\#END\#
~\\~\\
\#END EXAMPLE \\
-------------------------------------------------- \\
-------------------------------------------------- \\
The above is just an example, you do not have to follow this template exactly. \\
You can try different ways to describe the user's preferences, and you can also be more comprehensive in your description, if needed. \\
-------------------------------------------------- \\
Now it's your turn.     \\
Below is a list of items the user has rated: \\
 1. Title: {\color{red}\textbf{<item 1 title>}}\\
Description: {\color{red}\textbf{<item 1 description>}}\\
Price: {\color{red}\textbf{<item 1 price>}} \\
Review: {\color{red}\textbf{<item 1 review>}}
~\\~\\
2. Title: {\color{red}\textbf{<item 2 title>}}\\
Description: {\color{red}\textbf{<item 2 description>}}\\
Price: {\color{red}\textbf{<item 2 price>}} \\
Review: {\color{red}\textbf{<item 2 review>}}
~\\~\\
3. Title: {\color{red}\textbf{<item 3 title>}}\\
Description: {\color{red}\textbf{<item 3 description>}}\\
Price: {\color{red}\textbf{<item 3 price>}} \\
Review: {\color{red}\textbf{<item 3 review>}}
~\\~\\
4. Title: {\color{red}\textbf{<item 4 title>}}\\
Description: {\color{red}\textbf{<item 4 description>}}\\
Price: {\color{red}\textbf{<item 4 price>}} \\
Review: {\color{red}\textbf{<item 4 review>}}
~\\~\\
5. Title: {\color{red}\textbf{<item 5 title>}}\\
Description: {\color{red}\textbf{<item 5 description>}}\\
Price: {\color{red}\textbf{<item 5 price>}} \\
Review: {\color{red}\textbf{<item 5 review>}}
~\\~\\
6. Title: {\color{red}\textbf{<item 6 title>}}\\
Description: {\color{red}\textbf{<item 6 description>}}\\
Price: {\color{red}\textbf{<item 6 price>}} \\
Review: {\color{red}\textbf{<item 6 review>}}
~\\~\\
7. Title: {\color{red}\textbf{<item 7 title>}}\\
Description: {\color{red}\textbf{<item 7 description>}}\\
Price: {\color{red}\textbf{<item 7 price>}} \\
Review: {\color{red}\textbf{<item 7 review>}}
~\\~\\
8. Title: {\color{red}\textbf{<item 8 title>}}\\
Description: {\color{red}\textbf{<item 8 description>}}\\
Price: {\color{red}\textbf{<item 8 price>}} \\
Review: {\color{red}\textbf{<item 8 review>}}
~\\~\\
9. Title: {\color{red}\textbf{<item 9 title>}}\\
Description: {\color{red}\textbf{<item 9 description>}}\\
Price: {\color{red}\textbf{<item 9 price>}} \\
Review: {\color{red}\textbf{<item 9 review>}}
~\\~\\
10. Title: {\color{red}\textbf{<item 10 title>}}\\
Description: {\color{red}\textbf{<item 10 description>}}\\
Price: {\color{red}\textbf{<item 10 price>}} \\
Review: {\color{red}\textbf{<item 10 review>}} \\
--------------------------------------------------
   ~\\~\\   
-------------------------------------------------- \\
-------------------------------------------------- \\
Write a user profile to describe this user. \\
Ignore information that seems irrelevant or not informative. \\
Limit the profile to a maximum of {\color{red}166} words. \\
Even though the length of the example profile above might be different, there is a strict limit of {\color{red}166} words to your output.
~\\~\\
Output should be in the following format: \\
\#PROFILE\# <user\_profile> \#END\#
}

\subsection{Answer Agent}

{
\color{blue}
Below is a profile of a user's preferences:
~\\~\\
{\color{red}\textbf{<user profile (from DHR encoder)>}}
~\\~\\
--------------------------------------------------\\
Below are 6 items the user hasn't rated yet:
~\\~\\
1. Title: {\color{red}{\color{red}\textbf{<item 1 title>}}}\\
Description: {\color{red}{\color{red}\textbf{<item 1 description>}}}\\
Price: {\color{red}{\color{red}\textbf{<item 1 price>}}}
~\\~\\
2. Title: {\color{red}\textbf{<item 2 title>}}\\
Description: {\color{red}\textbf{<item 2 description>}}\\
Price: {\color{red}\textbf{<item 2 price>}}
~\\~\\
3. Title: {\color{red}\textbf{<item 3 title>}}\\
Description: {\color{red}\textbf{<item 3 description>}}\\
Price: {\color{red}\textbf{<item 3 price>}}
~\\~\\
4. Title: {\color{red}\textbf{<item 4 title>}}\\
Description: {\color{red}\textbf{<item 4 description>}}\\
Price: {\color{red}\textbf{<item 4 price>}}
~\\~\\
5. Title: {\color{red}\textbf{<item 5 title>}}\\
Description: {\color{red}\textbf{<item 5 description>}}\\
Price: {\color{red}\textbf{<item 5 price>}}
~\\~\\
6. Title: {\color{red}\textbf{<item 6 title>}}\\
Description: {\color{red}\textbf{<item 6 description>}}\\
Price: {\color{red}\textbf{<item 6 price>}}\\
--------------------------------------------------
~\\~\\
--------------------------------------------------\\
Below are a set of 6 questions about the items above:\\
(Q1) Rank the items {\color{red}\textbf{<id 5>}} and {\color{red}\textbf{<id 6>}} based on the user's preferences.\\
(Q2) Rank the items {\color{red}\textbf{<id 6>}} and {\color{red}\textbf{<id 2>}} based on the user's preferences.\\
(Q3) Rank the items {\color{red}\textbf{<id 2>}} and {\color{red}\textbf{<id 1>}} based on the user's preferences.\\
(Q4) Rank the items {\color{red}\textbf{<id 1>}} and {\color{red}\textbf{<id 3>}} based on the user's preferences.\\
(Q5) Rank the items {\color{red}\textbf{<id 3>}} and {\color{red}\textbf{<id 4>}} based on the user's preferences.\\
(Q6) Write a review for item {\color{red}\textbf{<id 1>}} as the user would write it.\\
--------------------------------------------------
~\\~\\          
Each of the 6 questions either asks you to rank two items based on the user's preferences, or to write a review for an item in the way the user might write it.\\
For ranking questions, a higher rank means the user would rate the item higher.\\
If there's a tie, pick the order randomly.\\
For review questions, you should write a raw review for the item as the user would write it.
~\\~\\
Formatting your answer:\\
Your output should be in the following format. For each ranking question you should output a line with its prediction.\\
For a review question, you should output a review as the user would write it for that item.
~\\~\\
For question k your output answer should be in the following format:\\
(Ak) \#PREDICTION\# [item\_id, item\_id] \#END\# \\
if it's a ranking quesiton, and \\
(Ak) \#PREDICTION\# user review \#END\# \\
if it's a review question.
~\\~\\
For the ranking question, the list is in order of the user's preferences for those item ids in that question. \\
For the review question, don't prefix it with anything else (like "User review"). Just write the raw review the user would write between the \#PREDICTION\# and \#END\#.
~\\~\\
For example, assume you are asked in question 1 to rank item ids 1 and 2, and you believe the user would rate item 1 higher than item 2, and in question 2 you're asked rank item ids 1, 3, where you the think the user would rate item 3 higher than item 1.
And assume in question 3 you are asked to write a review for item id 1, and let's say the user liked that item. \\
Then your complete output should be: \\
(A1) \#PREDICTION\# [1, 2] \#END\# \\
(A2) \#PREDICTION\# [3, 1] \#END\# \\
(A3) \#PREDICTION\# This shirt is great! I loved the color and the material. \#END\#
~\\~\\
Output should be exactly in the above format. Do not output anything else. \\
Remember to use \#PREDICTION\# and \#END\# for every answer.
}

\newpage
\section{Rater Study}
\label{appendix: rater study}

Below we show examples of rater forms including the contexts and questions:
\begin{mybox}{Sample Rater From for Prediction and Review Accuracy using Amazon Reviews}
\tiny
We have the following summary about preference of a user:
~\\~\\
This user favors comfortable, stylish women's clothing, particularly sandals and dresses. They appreciate value and are drawn to visually appealing designs, as evidenced by the high rating on the silver chain. A preference for easy-to-wear items is clear, with ratings of 4.0 on the sandals and dress. They seem to like simple, classic styles and are happy to purchase items within a reasonable price range. The 5/5 rating on the flip flops indicates a practical focus on comfort and style.
~\\~\\
1. Title: BIADANI Women Button Down Long Sleeve Basic Soft Knit Cardigan Sweater
Item Item description: "Casual, Elegant Fitted Knit Cardigan Sweaters That is Versatile and Made Out of High Quality Material and Luciously Soft."
~\\~\\
Price: \$10.01 - \$25.80
~\\~\\
--------------------------------------------------
~\\~\\
2. Title: Moxeay Women Sexy Sleeveless Spaghetti Strap Mini Club Dress
~\\~\\
Item Item description: "Material: Polyester (soft fabric make you comfortable to wear) Style: Spaghetti strap mini dress Features: U neck/V neck, Sleeveless, Spaghetti strap,Backless, High Waist, Back zipper Dress length: Mini length Occasion: Casual, Beach, Club, Prom, Banquet, Party Evening Attractive style and beautiful design, sexy summer sleeveless dress Ladies Dress ONLY, other accessories photographed not included. U neck style size: S( US 0) Bust:25.19\" Waist:24.40\" Length:23.62\" M(US 2) Bust:26.77\" Waist:25.98\" Length:24.01\" L(US 4) Bust:28.34\" Waist:27.55\" Length:24.40\" XL(US 6) Bust:31.49\" Waist:29.92\" Length:24.80\" Lace V neck style size: S:Bust 32.28\",Waist 27.55\", Length 30.31\" M:Bust 33.85\", Waist 29.13\", Length 30.70\" L:Bust 35.43\", Waist 30.70\", Length 31.10\" XL:Bust 36.22\", Waist 31.49\", Length 31.49\" Embroidery V neck style: S:Waist 29.92'',Bust 34.64'',Length 36.22'' M:Waist 31.49'',Bust 37.00'',Length 36.61'' L:Waist 33.07'',Bust 39.37'',Length 37.00'' XL:Waist 33.85'',Bust 41.73'',Length 37.40'' NOTE: 1.Size is Asian sizes,pls allow 1-2 inch size deviation due to manual measurement. 2.Colors may slightly different due to the lighting and monitor. Package Content: 1 x Women dress (Packed in Moxeay designed outpacking!)"
~\\~\\
Price: \$14.99 - \$21.99
~\\~\\
--------------------------------------------------
~\\~\\
3. Title: Miusol Women's Casual Flare Floral Contrast Evening Party Mini Dress
~\\~\\
Item description: unknown
~\\~\\
Price: unknown
~\\~\\
--------------------------------------------------
~\\~\\
4. Title: Damask Embossed Metal Business Card Case
~\\~\\
Item Item description: "This business card case features an embossed gold damask pattern on both sides. The case holds up to 10 business cards securely and snaps shut when not in use. Measures 3.75\" x 2.5\" x 0.25\"."
~\\~\\
Price: \$3.95
~\\~\\
--------------------------------------------------
~\\~\\
5. Title: Honolulu Jewelry Company Sterling Silver 1mm Box Chain Necklace, 14\" - 36\
~\\~\\
Item Item description: "Nickel free sterling silver 1mm box chain. Comes in different sizes with a spring clasp. Rhodium finished to prevent tarnishing. Made in Italy. Gift box included. From Honolulu Jewelry Company, Honolulu, Hawaii."
~\\~\\
Price: \$8.99
~\\~\\
--------------------------------------------------
~\\~\\
6. Title: Milumia Women's Button up Split Floral Print Flowy Party Maxi Dress
~\\~\\
Item description: unknown
~\\~\\
Price: \$18.99 - \$35.99
~\\~\\
--------------------------------------------------
~\\~\\
(Q1) Rank the items id 2 and id 3 based on the user's preferences.
~\\~\\
(Q2) Rank the items id 3 and id 6 based on the user's preferences.
~\\~\\
(Q3) Rank the items id 6 and id 5 based on the user's preferences.
~\\~\\
(Q4) Rank the items id 5 and id 4 based on the user's preferences.
~\\~\\
(Q5) Rank the items id 4 and id 1 based on the user's preferences.
~\\~\\
(Q6) Pick a review for item id 1 as the user would write it:
~\\~\\
REVIEW1:  Love,love,love the color of this cardi. Will be buying more from this seller. I also love the button detail on the sleeve/cuff. Would highly recommend!
~\\~\\
REVIEW2: This cardigan is so soft and comfy! It’s a great basic piece that goes with everything. The fit is perfect, and it’s really easy to throw on. I got it for a steal – it’s definitely worth the price.
\end{mybox}

\begin{mybox}{Sample Rater Form for Profile Consistency using Amazon Reviews}
\tiny
Given a user with the following purchase history:
~\\~\\
1. Title: Zumba Carpet Gliders for Shoes
~\\
Item Item description: "Dont get stuck on the carpet floor.Zumba Carpet Glidersallows you to step, shake, swivel and spin on carpet with ease and reduced risk of injury."
~\\
Price: unknown
~\\
User Rating: 1.0
~\\~\\
2. Title: Sterling Silver Leverback Earrings Black Pear Teardrop Made with Swarovski Crystals
~\\
Item Item description: "Black faceted crystal teardrops hang from sterling silver rings and leverback earwires. Solid 925 sterling silver, teardrops are approx 5/8 x 3/8 inches. Made with Swarovski Crystals. Gift box or organza bag included, color or style may vary. See Joyful Creations store for matching necklace."
~\\
Price: \$17.99
~\\
User Rating: 4.0
~\\~\\
3. Title: Wiipu fashion vintage luxurious pink color crystal brand designer statement women necklace(B379)
~\\
Item Item description: "Material:alloy rhinestone ,crystal Size: necklace ribbon chain is 35cm ,pendant is 17cm*12cm; shipping from China, usually take about 7--15days arrival,if you not accetp please don't order, thanks!"
~\\
Price: unknown
~\\
User Rating: 4.0
~\\~\\
4. Title: Kooljewelry Sterling Silver Bead and Diamond-Cut Ball Station Necklace (14, 16, 18, 20, 22, 24, 30, or 36 inch)
~\\
Item description: unknown
~\\
Price: \$25.99
~\\
User Rating: 5.0
~\\~\\
5. Title: Dearfoams Women's Sequin Flat Slipper
~\\
Item Item description: "Dearfoams is a slipper brand with great awareness. They understand exactly what customers are looking for-- from pampering your feet in cozy softness to keeping you one step ahead of the game."
~\\
Price: unknown
~\\~\\
User Rating: 4.0
~\\~\\
6. Title: Women's Chiffon Beachwear Dress Swimwear Bikini Cover-up Made in The USA
~\\
Item Item description: "SHORE TRENDZ quality constructed cover-up Chiffon dresses are created with lovely detail. Buy with confidence as they are MADE IN THE USA. These sexy Chiffon dresses have finished edges and lovely pattern detail! We appreciate you visiting and welcome you to check out our store at SHORE TRENDZ for more great items!!!"
~\\
Price: \$11.99
~\\
User Rating: 1.0
~\\~\\
7. Title: Dearfoams Women's Lurex Sweater Knit Ballerina Slipper
~\\
Item Item description: "This slipper features a cable knit upper with silver lurex, yarn pom embellishment with silver lurex, and elasticized throat line for secure fit. Brushed terry lining and insole, 10mm high density poly foam insole, with durable, skid resistant, TPR outsole.", "Dearfoams is a slipper brand with great awareness. They understand exactly what customers are looking for-- from pampering your feet in cozy softness to keeping you one step ahead of the game."
~\\
Price: unknown
~\\
User Rating: 3.0
~\\~\\
8. Title: Nine West Women's Able Synthetic Platform Pump
~\\
Item Item description: "Nine West offers a quick edit of the runways -- pinpointing the must have looks of the season, and translating what is fun, hip, and of the moment. It is trend-right footwear that you will reach for in your closet again and again. Nine West is sure to be your trusted resource for everyday chic style."
~\\
Price: \$29.99 - \$68.98
~\\
User Rating: 5.0
~\\~\\
9. Title: Annie Shoes Women's Devine Dress Pump
~\\
Item description: unknown
~\\
Price: unknown
~\\
User Rating: 4.0
~\\~\\
10. Title: TinkSky Wedding Tiara Rhinestones Crystal Bridal Headband Pageant Princess Crown
~\\
Item description: unknown
~\\
Price: \$8.99
~\\
User Rating: 3.0
~\\~\\
We want to use the following summary to capture user preference from the above purchase history:
\begin{enumerate}
\item \textbf{Brand Affinity:} Shows a strong preference for Nine West and Annie Shoes, evidenced by the 5-star ratings. A secondary interest in Sterling Silver jewelry.

\item \textbf{Style:} Appreciates fashionable slippers and pumps; likely enjoys comfortable yet stylish footwear. 

\item \textbf{Price Range:} Primarily purchases items in the $10 - $50 range.

\item \textbf{Negative Feedback:} The low rating of the beachwear dress and the tiara headband suggests a critical eye towards embellishments and potentially lower quality.

\item \textbf{Purchase type} Most likely female.
\end{enumerate}
Does the above summary faithfully capture user preference from their purchase history?
~\\~\\
1 - definite no \\
2 \\
3 \\
4 \\
5 - definite yes
\end{mybox}

\begin{mybox}{Sample Rater Form for Prediction Accuracy using MovieLens}
\tiny
We have the following summary about preference of a user:
~\\~\\
This user demonstrates a preference for comedies with a strong comedic impact, leaning towards the more witty and absurdist side.  They enjoy action and adventure elements woven into their comedic entertainment, showing an appreciation for thrill and humor combined.  While the user appreciates lightheartedness, they are not averse to a more serious tone - the presence of a thriller and a few dramas suggests a willingness to explore different genres, although they seem to lean towards more lighthearted plots. While the user appears to enjoy older comedies, they don't shy away from more contemporary fare, indicating a versatile taste across various time periods. They likely seek out movies with a clear comedic focus rather than those with a heavy dramatic weight.  The user isn't afraid to give lower ratings to movies they don't enjoy, suggesting a discerning taste and a desire for quality comedic entertainment.
~\\~\\
We will ask questions based on the following list of movies. Please do your own research (using IMDB) if you are not familiar with those movies.
~\\~\\
1. Title: Hangover, The (2009)
~\\~\\
--------------------------------------------------
~\\~\\
2. Title: Old Boy (2003)
~\\~\\
--------------------------------------------------
~\\~\\
3. Title: Sympathy for Mr. Vengeance (Boksuneun naui geot) (2002)
~\\~\\
--------------------------------------------------
~\\~\\
4. Title: Let the Right One In (Låt den rätte komma in) (2008)
~\\~\\
--------------------------------------------------
~\\~\\
5. Title: Spanking the Monkey (1994)
~\\~\\
--------------------------------------------------
~\\~\\
6. Title: Visitor Q (Bizita Q) (2001)
~\\~\\
--------------------------------------------------
~\\~\\
(Q1) Rank the movies id 2 and id 6 based on the user's preferences.
~\\~\\
(Q2) Rank the movies id 6 and id 5 based on the user's preferences.
~\\~\\
(Q3) Rank the movies id 5 and id 1 based on the user's preferences.
~\\~\\
(Q4) Rank the movies id 1 and id 3 based on the user's preferences.
~\\~\\
(Q5) Rank the movies id 3 and id 4 based on the user's preferences.
\end{mybox}

\begin{mybox}{Sample Rater Form for Profile Consistency using MovieLens}
\tiny
Given a user with the following history:
~\\~\\
1. Title: RoboCop (1987)
~\\~\\
User Rating: 3.5
~\\~\\
2. Title: Chasing Amy (1997)
~\\~\\
User Rating: 3.5
~\\~\\
3. Title: Grosse Pointe Blank (1997)
~\\~\\
User Rating: 3.0
~\\~\\
4. Title: Arachnophobia (1990)
~\\~\\
User Rating: 3.0
~\\~\\
5. Title: Mary Poppins (1964)
~\\~\\
User Rating: 1.5
~\\~\\
6. Title: Ice Age (2002)
~\\~\\
User Rating: 2.0
~\\~\\
7. Title: No Country for Old Men (2007)
~\\~\\
User Rating: 4.0
~\\~\\
8. Title: Wayne's World (1992)
~\\~\\
User Rating: 4.5
~\\~\\
9. Title: Bad Boys (1995)
~\\~\\
User Rating: 2.5
~\\~\\
10. Title: Planet of the Apes (1968)
~\\~\\
User Rating: 3.5
~\\~\\
We want to use the following summary to capture user preference from the above history:
~\\~\\
This user enjoys a mix of genres, primarily leaning towards action, comedy, and sci-fi.  They demonstrate a preference for films from the 80s and 90s and exhibit a decent tolerance for older movies. Their ratings suggest they favor films with a good balance of action, humor, and engaging plots.  The user is not averse to darker or more somber films as long as the storytelling is strong (e.g., No Country for Old Men). Their rating of 'Mary Poppins', however, shows a potential dislike for saccharine, overly-sentimental films.  The user shows definite interest in classic sci-fi but is less enthusiastic towards animation and may be reluctant to watch family-oriented content.
~\\~\\
Does the above summary faithfully capture user preference from their history?
~\\~\\
1 - definite no \\
2 \\
3 \\
4 \\
5 - definite yes
\end{mybox}

\newpage
\section{Qualitative Results}
\label{appendix: qualitative results}

\subsection{Amazon Profile Example \#1}

\paragraph{User History:}

\begin{enumerate}
\item \textbf{Title:} New Balance Women's WW665 Walking Shoe
\\\textbf{Item description:} "New Balance is dedicated to helping athletes achieve their goals. It's been their mission for more than a century to focus on research and development. It's why they don't design products to fit an image. They design them to fit. New Balance is driven to make the finest shoes for the same reason athletes lace them up: to achieve the very best.", "Get fit and stay chic with the WW665 walking shoe from New Balance. Ample mesh allows refreshing breathability while the cushy sole delivers comfort stride after stride. A sporty look and soft color palette make this an ideal find for your fitness routine."
\\\textbf{Price:} \$36.75 - \$62.39
\\\textbf{User Rating:} 5.0

\item \textbf{Title:} Dickies Women's Relaxed Fit Straight Leg Cargo Pant Fade \&Wrinkle Resistant
\\\textbf{Item description:} ``Inseam 32 inches. Fit tip: For accuracy, measure yourself in your undergarments. Give all measurements in inches. If your measurements are between sizes, order the larger size."
\\\textbf{Price:} \$28.99 - \$86.53
\\\textbf{User Rating:} 2.0

\item \textbf{Title:} Simulated Pink Pearl Rondelle Stretch Bracelet Silver Plated
Description: \\\textbf{Item description:} unknown
\\\textbf{Price:} \$12.99
\\\textbf{User Rating:} 5.0

\item \textbf{Title:} Eye Catching Women Leather Bracelet Silver Color Beads Cuff Jewelry with Magnetic Clasp 7.5\" (Black)
\\\textbf{Item description:} "Another eye catching Urban Jewelry bracelet,silver beads color make a stunning splash along a luxe multi midnight black leather bracelet with magnetic closure.Provides a touch of nature combined style... Guaranteed to add a unique, trendy touch to almost all outfitsShipping \&Delivery From The USAUrban Jewelry is located in New York City and Offer Worldwide ShippingEstimated Delivery Time: In the United States 2-4 Business Days. About Urban Jewelry We have a passion for fashion. Our goal is to create a jewelry haven where you will find great quality, affordable prices and trendy pieces. Urban Jewelry is an exclusive brand specializing in upscale stainlesssteel silver and leather accessories for women, men and teenagers. Urban Jewelry is locatedin New York City and ships worldwide. From the runway to your home. Urban Jewelry collection features thelatest styles, unique pieces, which will make you happy and your loved ones as well..Eye Catching Women Leather Bracelet Silver Color Beads Cuff Jewelry with Magnetic Stainless Steel Clasp 7.5\" (Black)"
\\\textbf{Price:} \$9.90
\\\textbf{User Rating:} 4.0

\item \textbf{Title:} Vikoros Women Flowy Lace Overlay Adjustable Strap Crop Top Tank Bustier
\\\textbf{Item description:} "Sell by Vikoros Store, pls refer the specification and picture details"
\\\textbf{Price:} unknown
\\\textbf{User Rating:} 2.0

\item \textbf{Title:} FRYE Women's Molly D Ring Short Boot
\\\textbf{Item description:} "Be the ring leader in the Molly D Ring Short boots by Frye\u00ae. Hammered full grain leather upper. Two buckle accents for a vintage look. Side zipper closure for easy on and off. Soft leather lining. Cushioned leather footbed for all-day comfort. Durable leather and rubber outsole for added traction. Imported. Measurements: Heel Height: 1 in Weight: 1 lb 1 oz Shaft: 5 1\u20442 in Product measurements were taken using size 8, width B - Medium. Please note that measurements may vary by size.", "The Frye Company is the oldest continuously operated shoe company in the United States. Founded in 1863 by John A. Frye, a well-to-do shoemaker from England, and family-run until 1945, Frye products have a long and illustrious history. Frye boots were worn by soldiers on both sides of America's Civil War, soldiers in the Spanish-American war, and by Teddy Roosevelt and his Rough Riders. When home-steading drew adventurous New England families to the West during the mid and late 1800's many of the pioneers wore Frye Boots for the long journey. Today Frye remains true to its roots with its line of heritage boots, but continues to innovate as it introduces chic new handbags, pumps, and sandals to its collection."
\\\textbf{Price:} \$179.00 - \$361.41
\\\textbf{User Rating:} 4.0

\item \textbf{Title:} totes Women's Zelus Snow Boot
\\\textbf{Item description:} "If you are looking for a premium fashionable cold weather boot, look no further than the Totes Zelus. A waterproof rubber shell keeps your feet warm and dry in all types of weather. Top of the line Metro fleece lined and a thick collar for that snuggly fit. Plus a front speed lace system and colored trim gives you the best of the best- The Totes Zelus"
\\\textbf{Price:} unknown
\\\textbf{User Rating:} 5.0

\item \textbf{Title:} totes Eric Black
\\\textbf{Item description:} "- Item: Totes Eric. - Upper: Polyurethane. - Style: Mens winter boot featuring a polyurethane upper - rear pull tab and full length zippers for easy entry - faux leather overlays - faux fur lining with padded upper - Thermolite insole. - Sole: Durable rubber lug outsole for traction. - Heel: 1 inch. - Lining:Fully faux fur lined for warmth. - Origin: Imported. - Color: Black."
\\\textbf{Price:} unknown
\\\textbf{User Rating:} 5.0

\item \textbf{Title:} Lonestar Jewelry Lacrosse Stick \&Necklace
\\\textbf{Item description:} "Exclusive Allan Robin Lacrosse Necklace with Lacrosse Stick Charm in Ball Chain Known to be the best gift for a Lacrosse player or enthusiast. Suprise your family with this perfect gift. Beautiful Lacrosse gift for girls and it comes in an elegant cotton-filled jewelry box so pretty easy to store. It's lead \& nickel free so no irritation to sensitive skin. They can wear it at any occassion or casually while doing day-to-day activities. This piece of modern fashion jewelry is unique enough to show your loved ones passion to the world, without being too flashy. A beautiful Lacrosse Stick Charm which Measures 1 1/4\" x 3/8\" Not Only Is This Lacrosse Jewelry A Great Gift From Parents or Grandparents, But Would Also Be A Wonderful Gesture From A Coach To Their Students! Order Yours Now And Dont Forget To Check Out The Hundreds Of Other Unique Sports Gifts That We Offer For All Of The Athletes In Your Life!"
\\\textbf{Price:} unknown
\\\textbf{User Rating:} 5.0

\item \textbf{Title:} M\&c Women's | Boyfriend Style \& Chronograph Look | FC0385
\\\textbf{Item description:} ``The Marciano \& Co. house brand M\&c offers customers superior value in an increasingly competitive market. The M\&c product lines draw inspiration from the most popular designs and are available at a fraction of designer prices. The Marciano \& Co. mission remains consistent; to provide customers with products that are novel in design and affordable without compromising quality."
\\\textbf{Price:} \$14.99
\\\textbf{User Rating:} 3.0
\end{enumerate}

\paragraph{User Profile:}

This user demonstrates varied shopping interests, with a clear preference for practical and comfortable footwear and accessories for themselves and potentially others. A recurring theme is a fondness for items emphasizing ease and warmth, evidenced by the 5-star ratings for snow boots (Totes Zelus \& Eric), walking shoes (New Balance), and a stretch bracelet (Simulated Pink Pearl).

While appreciating affordability (as revealed by purchases from brands like totes and M\&c), the user isn't averse to investing in quality when it aligns with their needs, indicated by the 4-star rating for Frye boots. Style leans towards sporty-casual and practical, with a tolerance for some embellishments (beaded bracelet, lacrosse necklace – likely a gift). Negative ratings suggest an unwillingness to purchase ill-fitting (Dickies pants) or poorly-performing (Vikorous top) items. An unclear gender preference outside of sale of women's apparel.

\subsection{Amazon Profile Example \#2}

\paragraph{User History:}

\begin{enumerate}
\item \textbf{Title:} Hand Crocheted Beanie (02)-Rasta \\ \textbf{Item description:} ``This beautifully thick, hand crocheted cotton beanie will brighten up your cold days with its unbeatable style and insulating warmth. The clean, round crown shape is achieved by 1 panel of tightly crochet cotton, of repeating black, green, yellow and red stripes. The base of the crown is made of an open crochet pattern with the band at the opening crocheted with elastic to give it a lot of stretch for a perfect fit around any head. This hat is available in other stripe patterns as well as in solid colors to meet your particular taste. Made of 100\% cotton. Hand crocheted item. Crown measures 6 deep, 7 wide. ONE SIZE fits most, from sizes 6 - 7 5/8. Available in an array of colors. Imported." \\ \textbf{Price:} \$4.99  \\ \textbf{User Rating:} 3.0 
\item \textbf{Title:} Intimo Men's Tricot Travel Pajama Set - Big Man Sizes \\ \textbf{Item description:}Every man should have a comfortable Nylon Travel Boxers. Wrinkle- Free." \\ \textbf{Price:} unknown  \\ \textbf{User Rating:} 5.0 
\item \textbf{Title:} Men Purple Mesh Pocket Shorts Inner Drawstring Avail Size S-5X Item description: unknown Price: unknown  \\ \textbf{User Rating:} 5.0 
\item \textbf{Title:} Breda Men's 1627-Gold Mitchell Multi Time Zone Watch \\ \textbf{Item description:} "A great-looking timepiece from Breda, this watch utilizes an excellent blend of materials and mechanics to create a functional accessory with a stylish flair. With the perfect mix of style and comfort, this watch will quickly become one of the most popular members of your watch collection.", "The Mitchel Collection", "This three time zone, large face men 2019s watch is available in three colors.", "Breda Watches", "Breda. Original style. Non-singular aesthetic.", "A creative collective with a shared appreciation for design that tells more than one story. We believe in the freedom of self-expression through style. Breda was born when we poured our collective imagination and expertise into designing watches inspired by a global lens. An eclectic unit of ambitious artists, designers, business-brains, photographers, innovators and style rebels, we've come together to create pieces that intrigue, inspire and challenge the expected.", "Our process of creation is our own unique alchemy. A key principle is that of discovery: we live contemporary culture, explore the past and dream up the future, scour the world's fashion stages and streets to challenge and inspire each other's imaginations. Then we design.", "We work with global partners to source the latest materials. With meticulous attention to detail, each innovative product is born with the purpose of transcending the ordinary." \\ \textbf{Price:} unknown  \\ \textbf{User Rating:} 5.0 
\item \textbf{Title:} Breda Men's 1627-silver Mitchell Multi Time Zone Watch \\ \textbf{Item description:}", "A great-looking timepiece from Breda, this watch utilizes an excellent blend of materials and mechanics to create a functional accessory with a stylish flair. With the perfect mix of style and comfort, this watch will quickly become one of the most popular members of your watch collection.", "The Mitchel Collection", "This three time zone, large face men\u2019s watch is available in three colors.", "Breda Watches", "Breda. Original style. Non-singular aesthetic.", "A creative collective with a shared appreciation for design that tells more than one story. We believe in the freedom of self-expression through style. Breda was born when we poured our collective imagination and expertise into designing watches inspired by a global lens. An eclectic unit of ambitious artists, designers, business-brains, photographers, innovators and style rebels, we've come together to create pieces that intrigue, inspire and challenge the expected.", "Our process of creation is our own unique alchemy. A key principle is that of discovery: we live contemporary culture, explore the past and dream up the future, scour the world's fashion stages and streets to challenge and inspire each other's imaginations. Then we design.", "We work with global partners to source the latest materials. With meticulous attention to detail, each innovative product is born with the purpose of transcending the ordinary." \\ \textbf{Price:} unknown  \\ \textbf{User Rating:} 2.0 
\item \textbf{Title:} Breda Men's 1627-black Mitchell Multi Time Zone Watch \\ \textbf{Item description:}", "A great-looking timepiece from Breda, this watch utilizes an excellent blend of materials and mechanics to create a functional accessory with a stylish flair. With the perfect mix of style and comfort, this watch will quickly become one of the most popular members of your watch collection.", "The Mitchel Collection", "This three time zone, large face men\u2019s watch is available in three colors.", "Breda Watches", "Breda. Original style. Non-singular aesthetic.", "A creative collective with a shared appreciation for design that tells more than one story. We believe in the freedom of self-expression through style. Breda was born when we poured our collective imagination and expertise into designing watches inspired by a global lens. An eclectic unit of ambitious artists, designers, business-brains, photographers, innovators and style rebels, we've come together to create pieces that intrigue, inspire and challenge the expected.", "Our process of creation is our own unique alchemy. A key principle is that of discovery: we live contemporary culture, explore the past and dream up the future, scour the world's fashion stages and streets to challenge and inspire each other's imaginations. Then we design.", "We work with global partners to source the latest materials. With meticulous attention to detail, each innovative product is born with the purpose of transcending the ordinary." \\ \textbf{Price:} unknown  \\ \textbf{User Rating:} 3.0 
\item \textbf{Title:} Camouflage Cap, Color: Orange Camo, Size: One Size Item description: unknown Price: \$8.66  \\ \textbf{User Rating:} 5.0 
\item \textbf{Title:} LRG Men's Big-Tall Core Collection Nine Camo T-Shirt \\ \textbf{Item description:}The Core Collection Nine Camo Tee is constructed of Jersey featuring a contrasting pattern and custom camo print expertly rendered throughout with custom tag on neckline and hem." \\ \textbf{Price:} unknown  \\ \textbf{User Rating:} 5.0 
\item \textbf{Title:} Rocawear Men's Big-Tall Gradient Black Hearts Item description: unknown Price: unknown  \\ \textbf{User Rating:} 5.0 
\item \textbf{Title:} Kangol Men's Tropic Galaxy Cap \\ \textbf{Item description:}The galaxy is a kangol interpretation of a classic 8-quarter cap" \\ \textbf{Price:} \$37.40  \\ \textbf{User Rating:} 5.0

\end{enumerate}

\paragraph{User Profile:}

This user demonstrates a strong preference for men’s apparel, particularly casual wear. They consistently rate items favorably (mostly 5-star) suggesting a generally positive shopping experience and valuing style. Key interests include:

\begin{itemize}
\item \textbf{Streetwear \& Camouflage:} A recurring theme of camouflage patterns and brands like LRG and Rocawear indicates a liking for streetwear aesthetics.
\item \textbf{Caps \& Hats:} Several purchases of caps (Kangol, camouflage, orange camo) indicate a strong interest in headwear.
\item \textbf{Timepieces:} The purchase of the Breda watches, despite a mixed rating experience reveals an interest in watches.
\item \textbf{Comfortable Loungewear:} The purchase of pajama sets highlights a preference for comfort.
\item \textbf{Price Sensitivity:} The moderate rating on the beanie suggests a mindful of price.
\end{itemize}
The user appears to be male, likely younger, and influenced by contemporary trends. They are open to experimenting with different styles while maintaining a casual, streetwear-inspired aesthetic.

\subsection{Amazon Profile Example \#3}

\paragraph{User History:}

\begin{enumerate}
\item \textbf{Title:} Rubik's Cube - Mens Hip To Be Square Soft T-shirt Small Off-white \\ \textbf{Item description:} ``From Rubiks Cube comes this soft cotton short sleeve T-Shirt in off-white that features The Rubiks Cube and says Hip To Be Square. Great T-Shirt for anyone who loves The Rubiks Cube." \\ \textbf{Price:} unknown \\ \textbf{User Rating:} 3.0 
\item \textbf{Title:} Intimo Women's Printed Microfleece Pajama Pant \\ \textbf{Item description:} unknown \\ \textbf{Price:} \$18.00 \\ \textbf{User Rating:} 5.0 
\item \textbf{Title:} Minecraft Boys' Adventure Youth Tee \\ \textbf{Item description:} ``The first time you find yourself staring out at the expansive world of blocks laid out before you, the possibilities are truly endless. What will you choose to do in this biome, generated just for you you can build soaring castles, dig sprawling underground complexes, or maybe you'll set out to see what is just over the horizon. Whichever way you choose to play, each adventure will be your very own." \\ \textbf{Price:} \$14.50 - \$28.99 \\ \textbf{User Rating:} 4.0 
\item \textbf{Title:} Hollywood Star Fashion Casual Basic Women's Semi-Crop Camisole Cami Tank Top with Adjustable Straps \\ \textbf{Item description:} ``This is a long length Tank top Versatile Basic Spaghetti Strap Satin Trim Stretch Camisole Tank Yoga Everyday Active Adventure Travel Fitted Scoop neckline, adjustable shoulder straps Satin Trim Fully stretchable Please note: this top dose NOT feature a built-in shelf bra Body length in size medium: 28'' 95\% Cotton, 5\% Spandex Imported. Satisfaction guaranteed Returns accepted. We ship worldwide" \\ \textbf{\\ \textbf{Price:}:} \$5.00 - \$19.99 \\ \textbf{User Rating:} 5.0 
\item \textbf{Title:} MANDI HOME Hot Sale Wedding Fashion 925 Silver Plated Jewelry Set Big Hand Chain Bracelet Necklace Ring Stud Earings Eardrop Water Drops \\ \textbf{Item description:} ``1pcs for each!" \\ \textbf{Price:} \$6.35 \\ \textbf{User Rating:} 2.0 
\item \textbf{Title:} Fruit of the Loom Girls' Cotton Spaghetti Strap Sport Bra \\ \textbf{Item description:} ``Fruit of the Loom girls' spaghetti strap sport bras provide comfort and support for any activity. The cotton and lycra materials give them a soft feel and also make them machine-washable. Give the girl in your life a good start with Fruit of the Loom. Designed to give her support while giving her room to grow, tried and true Fruit of the Loom underclothes are the perfect choice for your growing girl" \\ \textbf{Price:} \$7.53 - \$12.68 \\ \textbf{User Rating:} 5.0 
\item \textbf{Title:} Fruit Of The Loom Women's Breathable Underwear Multipack (Assorted) \\ \textbf{Item description:} ``Our breathable cotton collection is designed with a soft, cotton-mesh fabrication that not only supports the flow of air but wicks away moisture to help you stay cool and dry all day long." \\ \textbf{Price:} \$4.50 - \$45.00 \\ \textbf{User Rating:} 5.0 
\item \textbf{Title:} Froomer Womens Winter Outwear Light Down Coat Hooded Jackets \\ \textbf{Item description:} unknown \\ \textbf{Price:} unknown \\ \textbf{User Rating:} 5.0 
\item \textbf{Title:} Rockland Luggage 17 Inch Rolling Backpack, Pink Dot, Medium \\ \textbf{Item description:} ``Rolling backpack with double skate wheels. Measures 13 inch lx10 inch wx17 inch h. Features include 2 zippered pockets on the front of the bag with one having a built in organizer. Zippered main compartment opens to large storage opening, and two mesh side pockets. Fabric handle at the top and telescoping pull handle. Bag made of durable polyester with metal telescoping handle and heavy duty plastic stand on bottom to keep the bag standing.", "R01-PINKGIRAFFE shown.", "", "", "", "" \\ \textbf{Price:} \$3.24 \\ \textbf{User Rating:} 5.0 
\item \textbf{Title:} Bags for Less Large Clear Vinyl Tote Bags Shoulder Handbag (Royal Blue) \\ \textbf{Item description:} unknown

\end{enumerate}

\paragraph{User Profile:}

This user demonstrates a preference for affordable, practical, and comfortable clothing and accessories, primarily for women and children. While they appreciate value, ratings suggest a focus on functionality over high-end brands, as evidenced by consistently positive reviews for items like Fruit of the Loom underwear and Fruit of the Loom sport bras. The 5-star rating for a Rockland luggage rolling backpack indicates an appreciation for travel-friendly and durable goods.

They show an interest in basic apparel and accessories (camisoles, pajamas). There's a slight openness to novelty items (Rubik's Cube shirt, Minecraft tee), but overall, their choices lean towards everyday essentials. The lower rating on the jewelry set suggests a less enthusiastic response to more elaborate items. The user appears to gravitate toward comfort, utility, and reasonable pricing.

\subsection{Amazon Profile Example \#4}

\paragraph{User History:}

\begin{enumerate}

\item \textbf{Title:} Naturalizer Women's Ringo Sandal \\ \textbf{Item description:} ``The Ringo is a slingback sandal that features an N5 Comfort System and a manmade outsole.", "Naturalizer was one of the first shoe brands that women could turn to for the feminine style they coveted and the comfort they thought was impossible to attain. Naturalizer's fresh, unpretentious designs are a smooth fit with your wardrobe, your life and your own unique style. Naturalizer promises style that makes you look good and feel good - always." \\ \textbf{Price:} unknown \\ \textbf{User Rating:} 4.0 

\item \textbf{Title:} Aerosoles Women's Tapestry \\ \textbf{Item description:} ``Every girl needs a timeless pump like the Aerosoles Tapestry. Featuring a 3 covered heel, softly angled toe and classic lines for a tried-and-true look. Stunningly soft memory foam insole is stitched and cushioned for your comfort, while the flexible rubber sole with diamond pattern drinks up hard impact. Youll feel great all day long!", "Destined to be your new favorite, Tapestry from Aerosoles Women's offers professional polish for the office and beyond. Showcasing a classic silhouette, this pretty pump features a leather or fabric upper that slips on to reveal the unbelievable comfort from Aerosoles that you have come to know and love. A flexible rubber outsole and a modest heel add a tasteful touch to any ensemble." \\ \textbf{Price:} \$64.99 \\ \textbf{User Rating:} 2.0 

\item \textbf{Title:} Eagle Creek Travel Gear Undercover Money Belt (Khaki) \\ \textbf{Item description:} ``Looking for money belts? Look no further than this simple waist-worn under-clothing solution. Keep important travel documents and personal identification items out of sight in this money belt. Its made of durable and lightweight rip-stop fabric with a moisture-wicking and breathable back panel. Complete with zippered pocket for secure organization and soft elastic waistband with strap keeper. When youre not wearing it, simply tuck the strap into the slip pocket on the back, which was conveniently created for waist strap storage. Travel solutions that make sense." \\ \textbf{Price:} \$15.85 \\ \textbf{User Rating:} 5.0 

\item \textbf{Title:} Maidenform Women's Comfort Devotion Demi Bra \\ \textbf{Item description:} ``Magnificently smooth and supportive. Maidenform's Comfort Devotion Demi Bra features foam, contour underwire cups made of plush fabric and smoothing wings with super soft fabric on the inside.Line dry or lay flat to dry" \\ \textbf{Price:} \$15.80 - \$94.21 \\ \textbf{User Rating:} 5.0 

\item \textbf{Title:} Champion Women's Jersey Pant \\ \textbf{Item description:} ``Champion Jersey Pant with a rib waistband is just the right fit and look for everyday wear." \\ \textbf{Price:} \$13.20 - \$72.06 \\ \textbf{User Rating:} 3.0 

\item \textbf{Title:} Men's Cotton Casual Ankle Socks \\ \textbf{Item description:} ``", "" \\ \textbf{Price:} \$15.90 \\ \textbf{User Rating:} 5.0

\item \textbf{Title:} Champion Absolute Sports Bra With SmoothTec Band \\ \textbf{Item description:} ``The Absolute workout bra solids and prints at a great value. This bra has a patented smooth tec band for the ultimate in chafe resistance and comfort. A must have for any gym bag." \\ \textbf{Price:} \$7.93 - \$48.00 \\ \textbf{User Rating:} 5.0 

\item \textbf{Title:} uxcell Men Point Collar Button Down Long Sleeves Plaid Detail Slim Fit Shirts \\ \textbf{Item description:} ``Description:One mock pocket point collar button down long sleeves plaid detail slim fit shirt. Feature mock pocket, plaids detail for build up your special character. Buttoned point collar for standard button down shirt. Buttoned cuffs is fused to keep a crisp, dressy appearance. Soft touch, comfortable fabric which is comfort to wear in all season. Suitable for date, daily work, travel and everyday wear. Match with formal trousers or stylish denim pants to build up fashion casual look. Please check your measurements to make sure the item fits before ordering. Body Size Chart (in inches)International \\ \textbf{Price:} \$12.24 - \$19.81 \\ \textbf{User Rating:} 2.0 

\item  \textbf{Title:} Mens Colorful Dress Socks Argyle - HSELL Men Multicolored Argyle Pattern Fashionable Fun Crew Socks Item description: unknown \\ \textbf{Price:} \$11.99 \\ \textbf{User Rating:} 5.0 

\item \textbf{Title:} Marvel Comics Crew Socks Mens Mix Hulk Spiderman, The Avengers Large With Bonus One XWi Wrist Sweatband \\ \textbf{Item description:} ``A simply Marvel-ous way to get your man to change socks regularly, each set of socks feature classic artwork from the comics themselves, with full color, detailed woven images of the likes of Avengers past and present, including The Hulk, Spider-Man, Iron Man and Captain America Designs: Assorted Contents: 98\% polyester, 2\% elastane Our guarantee: 100\% official merchandise Care instructions: Machine washable Packaging: Hanger Brand: Marvel Comics Size: USA(7-12) UK(6-11) Includes One Xportsworld Innovations Wrist Sweatband (Color chosen randomly)Composition: 80\% Polyester, 20\% Cotton Size: Width 8cm x Length 9cm ( W 3.14In x L 3.54In)" \\ \textbf{Price:} unknown \\ \textbf{User Rating:} 4.0

\end{enumerate}

\paragraph{User Profile:}

This user demonstrates a practical and value-focused shopping style with varied interests. They show a clear preference for comfortable and practical items, consistently rating items highly that prioritize these aspects. A strong inclination towards athletic wear and socks (Champion, Marvel socks) is evident, suggesting an active lifestyle or interest in sports-related apparel. They appreciate quality, as shown by the 4-star rating for Naturalizer sandals.

While primarily focused on apparel, the purchase of a money belt suggests travel interests. The mixed reviews (2-star for Aerosoles, 2-star for uxcell shirt) indicate potential sensitivity to fit or style outside their comfort zone. Maidenform bra and women's socks suggest occasional purchases of items for a female partner. The user appears to find good deals and prioritizes functionality over high-fashion trends.

\subsection{MovieLens Example \#1}

\paragraph{User History:}

\begin{enumerate}
\item \textbf{Title:} Star Wars: Episode IV - A New Hope (1977)\\
\textbf{User Rating:} 5.0\\
--------------------------------------------------
\item \textbf{Title:} Star Wars: Episode V - The Empire Strikes Back (1980)\\
\textbf{User Rating:} 5.0\\
--------------------------------------------------
\item \textbf{Title:} Raiders of the Lost Ark (Indiana Jones and the Raiders of the Lost Ark) (1981)\\
\textbf{User Rating:} 4.5\\
--------------------------------------------------
\item \textbf{Title:} Matrix, The (1999)\\
\textbf{User Rating:} 5.0\\
--------------------------------------------------
\item \textbf{Title:} Aliens (1986)\\
\textbf{User Rating:} 3.0\\
--------------------------------------------------
\item \textbf{Title:} Usual Suspects, The (1995)\\
\textbf{User Rating:} 4.0\\
--------------------------------------------------
\item \textbf{Title:} Léon: The Professional (a.k.a. The Professional) (Léon) (1994)\\
\textbf{User Rating:} 4.0\\
--------------------------------------------------
\item \textbf{Title:} Star Wars: Episode VI - Return of the Jedi (1983)\\
\textbf{User Rating:} 5.0\\
--------------------------------------------------
\item \textbf{Title:} Dark Knight, The (2008)\\
\textbf{User Rating:} 3.5\\
--------------------------------------------------
\item \textbf{Title:} Godfather, The (1972)\\
\textbf{User Rating:} 4.0
\end{enumerate}

\paragraph{User Profile:}

This user enjoys intellectually stimulating and visually impressive films, favoring narratives with depth and complex characters. They appreciate cinematic excellence and are willing to engage with challenging themes, demonstrating a sophisticated taste. A strong preference exists for science fiction and action, particularly those with philosophical underpinnings, as evidenced by high ratings of “Star Wars” films and “The Matrix”.

They appreciate well-crafted thrillers and crime dramas, like “The Usual Suspects” and “Léon: The Professional”, valuing intricate plots and suspenseful storytelling. Classic films, notably “The Godfather”, are also enjoyed, suggesting an appreciation for cinematic history. 

The user isn't limited to specific genres, embracing a diverse range of eras and cultural influences, with a noticeable fondness for European cinema. They are comfortable with slower pacing and unconventional structures, valuing atmosphere and character development. Overall, this individual seeks movies that provoke thought and deliver lasting impressions.

\subsection{MovieLens Example \#2}

\paragraph{User History:}

\begin{enumerate}
\item \textbf{Title:} Last of the Mohicans, The (1992) \\
\textbf{User Rating:} 3.0 \\
--------------------------------------------------
\item \textbf{Title:} Ghostbusters (a.k.a. Ghost Busters) (1984) \\
\textbf{User Rating:} 4.0 \\
--------------------------------------------------
\item \textbf{Title:} Star Wars: Episode I - The Phantom Menace (1999) \\
\textbf{User Rating:} 5.0 \\
--------------------------------------------------
\item \textbf{Title:} Police Academy 3: Back in Training (1986) \\
\textbf{User Rating:} 3.0 \\
--------------------------------------------------
\item \textbf{Title:} E.T. the Extra-Terrestrial (1982) \\
\textbf{User Rating:} 3.0 \\
--------------------------------------------------
\item \textbf{Title:} 13th Warrior, The (1999) \\
\textbf{User Rating:} 4.0 \\
--------------------------------------------------
\item \textbf{Title:} Being John Malkovich (1999) \\
\textbf{User Rating:} 4.0 \\
--------------------------------------------------
\item \textbf{Title:} Blair Witch Project, The (1999) \\
\textbf{User Rating:} 3.0 \\
--------------------------------------------------
\item \textbf{Title:} Bug's Life, A (1998) \\
\textbf{User Rating:} 4.0 \\
--------------------------------------------------
\item \textbf{Title:} Clerks (1994) \\
\textbf{User Rating:} 5.0
\end{enumerate}

\paragraph{User Profile:}

This user demonstrates a preference for films that are both engaging and thought-provoking. They appreciate a balance between entertaining narratives and deeper themes, often favoring stories with complex characters and layered plots. A strong interest in both action and comedy is evident, alongside a clear appreciation for science fiction and suspenseful thrillers.

The user enjoys films with strong directorial vision and is not afraid to explore unconventional storytelling techniques. They value realistic characters and are drawn to films that delve into the human condition, often with a focus on relationships and identity. 

Despite a love for contemporary cinema, they also appreciate classic films, signaling a broad appreciation for cinematic history and diverse cultural influences. They are open to a range of emotional tones, from lighthearted to dramatic, and demonstrate a tolerance for challenging subject matter. Ultimately, this user seeks out movies that are memorable, intellectually stimulating, and leave a lasting impression.

\subsection{MovieLens Example \#3}

\paragraph{User History:}

\begin{enumerate}
\item \textbf{Title:} Batman (1989) \\
\textbf{User Rating:} 4.0 \\
--------------------------------------------------
\item \textbf{Title:} Apollo 13 (1995) \\
\textbf{User Rating:} 3.0 \\
--------------------------------------------------
\item \textbf{Title:} Pulp Fiction (1994) \\
\textbf{User Rating:} 5.0 \\
--------------------------------------------------
\item \textbf{Title:} True Lies (1994) \\
\textbf{User Rating:} 5.0 \\
--------------------------------------------------
\item \textbf{Title:} Die Hard: With a Vengeance (1995) \\
\textbf{User Rating:} 5.0 \\
--------------------------------------------------
\item \textbf{Title:} Aladdin (1992) \\
\textbf{User Rating:} 3.0 \\
--------------------------------------------------
\item \textbf{Title:} Ace Ventura: Pet Detective (1994) \\
\textbf{User Rating:} 4.0 \\
--------------------------------------------------
\item \textbf{Title:} Batman Forever (1995) \\
\textbf{User Rating:} 5.0 \\
--------------------------------------------------
\item \textbf{Title:} Shawshank Redemption, The (1994) \\
\textbf{User Rating:} 3.0 \\
--------------------------------------------------
\item \textbf{Title:} Fugitive, The (1993) \\
\textbf{User Rating:} 4.0
\end{enumerate}

\paragraph{User Profile:}

This user enjoys action-packed, engaging films with a strong sense of adventure and spectacle. They have a pronounced taste for thrillers and superhero movies, evident in their high ratings for iconic films like \textbf{Batman} and \textbf{Die Hard}. A preference for visually impressive and exciting narratives is clear. 

They appreciate well-constructed plots and fast-paced storytelling, enjoying films that deliver immediate entertainment. However, the user also values character-driven narratives, demonstrated by the consistent 4.0 rating for \textbf{The Shawshank Redemption}. 

There’s a clear appreciation for entertainment that is accessible and exciting. They generally gravitate towards movies with a positive tone and fun energy, similar to \textbf{Aladdin} and \textbf{Ace Ventura}.  The user isn’t overly concerned with complex themes, valuing thrills and excitement over deeper philosophical explorations. A preference for blockbuster entertainment is evident, alongside a fondness for the superhero genre.

\subsection{MovieLens Example \#4}

\paragraph{User History:}

\begin{enumerate}
\item \textbf{Title:} Godfather: Part III, The (1990) \\
\textbf{User Rating:} 3.0 \\
--------------------------------------------------
\item \textbf{Title:} Stepmom (1998) \\
\textbf{User Rating:} 3.0 \\
--------------------------------------------------
\item \textbf{Title:} Blair Witch Project, The (1999) \\
\textbf{User Rating:} 3.0 \\
--------------------------------------------------
\item \textbf{Title:} Haunting, The (1963) \\
\textbf{User Rating:} 5.0 \\
--------------------------------------------------
\item \textbf{Title:} Abbott and Costello Meet Frankenstein (1948) \\
\textbf{User Rating:} 3.0 \\
--------------------------------------------------
\item \textbf{Title:} Black Sabbath (Tre volti della paura, I) (1963) \\
\textbf{User Rating:} 5.0 \\
--------------------------------------------------
\item \textbf{Title:} General, The (1926) \\
\textbf{User Rating:} 5.0 \\
--------------------------------------------------
\item \textbf{Title:} Wings (1927) \\
\textbf{User Rating:} 4.0 \\
--------------------------------------------------
\item \textbf{Title:} Jail Bait (1954) \\
\textbf{User Rating:} 3.0 \\
--------------------------------------------------
\item \textbf{Title:} M (1931) \\
\textbf{User Rating:} 5.0
\end{enumerate}

\paragraph{User Profile:}

This user enjoys films with a strong narrative and appreciates classic cinema, particularly those from the mid-20th century. They favor dramas and thrillers, often with complex characters and thought-provoking themes, demonstrating a sophisticated taste. A significant portion of their ratings fall within the 3.0 – 5.0 range, indicating an ability to appreciate both well-executed and critically acclaimed movies.

They are open to diverse genres and cultural origins, showing a preference for European and international films. While they enjoy suspenseful and eerie stories like “The Haunting,” they also appreciate lighter, comedic fare like “Abbott and Costello Meet Frankenstein.” The user values strong directorial vision and cinematic artistry, favoring films that build atmosphere and explore deeper human experiences. 

A tolerance for slower pacing and unconventional storytelling is evident. They’re not afraid to give a lower rating to films that don’t meet their expectations.  The user consistently seeks out films that are visually engaging and emotionally resonant, demonstrating a deep appreciation for the art of filmmaking.

\end{appendix}

\end{document}